\documentclass{article}

\usepackage{microtype}
\usepackage{graphicx}
\usepackage{subcaption}
\usepackage{booktabs} % for professional tables
\usepackage{multirow}
\usepackage{soul}      % for \hl{}
\usepackage{tcolorbox} % for prompt box

\usepackage{xcolor}
\usepackage{tcolorbox}
\usepackage{tikz}

\usepackage{hyperref}

\usepackage[preprint]{icml2026}

\usepackage{amsmath}
\usepackage{amssymb}
\usepackage{mathtools}
\usepackage{amsthm}

\usepackage[capitalize,noabbrev]{cleveref}

\usepackage{amsmath,amsfonts,bm}

\def\eqref#1{equation~\ref{#1}}
\def\1{\bm{1}}
\DeclareMathAlphabet{\mathsfit}{\encodingdefault}{\sfdefault}{m}{sl}
\SetMathAlphabet{\mathsfit}{bold}{\encodingdefault}{\sfdefault}{bx}{n}

\usepackage{amsmath}

\newcommand{\x}{x}

\newcommand{\axi}[1][i]{\x_{#1}}

\newcommand{\pseqxn}[1][n]{(\axi[i])_{i\geq 1}} % sequence with parentheses
\newcommand{\pseqxnn}[1][n]{(\axi[i])_{i=1}^n} % sequence with parentheses

\newcommand{\Prob}{\ensuremath{{\mathbb{P}}}}

\def\balign#1\ealign{\begin{align}#1\end{align}}
\def\baligns#1\ealigns{\begin{align*}#1\end{align*}}
\def\balignat#1\ealign{\begin{alignat}#1\end{alignat}}
\def\balignats#1\ealigns{\begin{alignat*}#1\end{alignat*}}
\def\bitemize#1\eitemize{\begin{itemize}#1\end{itemize}}
\def\benumerate#1\eenumerate{\begin{enumerate}#1\end{enumerate}}

\newenvironment{talign*}
 {\csname align*\endcsname}
 {\endalign}
\newenvironment{talign}
 {\csname align\endcsname}
 {\endalign}

\def\balignst#1\ealignst{\begin{talign*}#1\end{talign*}}
\def\balignt#1\ealignt{\begin{talign}#1\end{talign}}
\newcommand{\qtext}[1]{\quad\text{#1}\quad}

\let\originalleft\left
\let\originalright\right
\renewcommand{\left}{\mathopen{}\mathclose\bgroup\originalleft}
\renewcommand{\right}{\aftergroup\egroup\originalright}

\def\tinycitep*#1{{\tiny\citep*{#1}}}
\def\tinycitealt*#1{{\tiny\citealt*{#1}}}
\def\tinycite*#1{{\tiny\cite*{#1}}}
\def\smallcitep*#1{{\scriptsize\citep*{#1}}}
\def\smallcitealt*#1{{\scriptsize\citealt*{#1}}}
\def\smallcite*#1{{\scriptsize\cite*{#1}}}

\def\mbf#1{\mathbf{#1}}
\def\mbb#1{\mathbb{#1}}
\def\mc#1{\mathcal{#1}}
\def\mrm#1{\mathrm{#1}}

\def\<{\left\langle} % Angle brackets
\def\>{\right\rangle}

\def\implies{\quad\Longrightarrow\quad}
\def\E{\mbb{E}} % Expectation symbol

\providecommand{\argmin}{\mathop\mathrm{arg min}}

\theoremstyle{plain}
\newtheorem{theorem}{Theorem}[section]

\newtheorem{lemma}[theorem]{Lemma}

\theoremstyle{definition}

\theoremstyle{remark}
\newtheorem{remark}[theorem]{Remark}

\usepackage[textsize=tiny]{todonotes}

\icmltitlerunning{GOPO: Policy optimization using ranked rewards}

\begin{document}

\twocolumn[
  \icmltitle{GOPO: Policy Optimization using Ranked Rewards}



  \icmlsetsymbol{equal}{$\star$}

  \begin{icmlauthorlist}
    \icmlauthor{Kyuseong Choi}{equal,yyy}
    \icmlauthor{Dwaipayan Saha}{equal,xxx}
    \icmlauthor{Woojeong Kim}{equal,yyy}
    \icmlauthor{Anish Agarwal}{xxx}
    \icmlauthor{Raaz Dwivedi}{yyy}
  \end{icmlauthorlist}

  \icmlaffiliation{yyy}{Cornell Tech, Cornell University, NY, United States}
  \icmlaffiliation{xxx}{Columbia University, NY, United States}

  \icmlcorrespondingauthor{Kyuseong Choi}{kc728@cornell.edu}

  \icmlkeywords{}

  \vskip 0.3in
]



\printAffiliationsAndNotice{}  % no special notice (required even if empty)

\begin{abstract}
  % Standard reinforcement learning from human feedback (RLHF) trains a reward model on pairwise preference data and then uses it for policy optimization. However, while the reward model is optimized for relative ranking, existing policy optimization techniques use absolute rewards during training. For settings where the rewards are non-verifiable—e.g., summarization, instruction following and chat completion—this misalignment often leads to suboptimal performance. We introduce GOPO, a policy optimization method that uses only the ranking of the rewards and discards their magnitudes. Our rank-based transformation of rewards provides several gains, compared to GRPO, in settings with non-verifiable rewards: (1) consistently higher training reward trajectories, (2) superior test evaluations at its best checkpoint, and (3) a fast convergence to its best checkpoint. We demonstrate the dominance of GOPO over GRPO across a suite of tasks and model sizes.
  Standard reinforcement learning from human feedback (RLHF) trains a reward model on pairwise preference data and then uses it for policy optimization. However, while reward models are optimized to capture relative preferences, existing policy optimization techniques rely on absolute reward magnitudes during training. In settings where the rewards are non-verifiable—such as summarization, instruction following, and chat completion—this misalignment often leads to suboptimal performance. We introduce \emph{Group Ordinal Policy Optimization}~(GOPO), a policy optimization method that uses only the ranking of the rewards and discards their magnitudes. Our rank-based transformation of rewards provides several gains, compared to Group Relative Policy Optimization (GRPO), in settings with non-verifiable rewards: (1) consistently higher training/validation reward trajectories, (2) improved LLM-as-judge evaluations across most intermediate training steps, and (3) reaching a policy of comparable quality in substantially less training steps than GRPO. We demonstrate consistent improvements across a range of tasks and model sizes. Our source code is available at \url{https://github.com/friendshipkim/gopo}.
  % Our source code is available at \url{https://anonymous.4open.science/r/gopo-1BFA/README.md}.
\end{abstract}

\newcommand{\rankpolicy}{\pi_{\mrm{rank}}}
\newcommand{\originalpolicy}{\pi_{\mrm{grpo}}}
\newcommand{\pizero}{\pi_{\mrm{ref}}}
\newcommand{\piold}{\pi_{\mrm{old}}}
\newcommand{\pinew}{\pi_{\mrm{new}}}

% \textcolor{blue}{\textbf{TODO}
% \begin{itemize}
%     \item Experiments with Qwen policy model family (\texttt{Qwen3-7B}) using \texttt{UltraChat} dataset and \texttt{TLDR} using \texttt{QRM} (\texttt{Llama-8B} based) and \texttt{SkyWalk} (\texttt{Qwen3-8B} based) reward models respectively.
%     % \item Softmax on the rewards. This is an opposite extreme we can apply on the z-scored rewards 
%     \item Need to think how to visualize our results (figures for proposed method \& numerical results)
% \end{itemize}
% }

\section{Introduction}

% \textcolor{blue}{KC: PPO uses Reward models for non-verifiable tasks as is }

\newcommand{\pigopo}{\pi_{\mrm{gopo}}^\star}
\newcommand{\pigrpo}{\pi_{\mrm{grpo}}^\star}

\iffalse
\begin{itemize}
    \item Significance of policy optimization (PO) in literature,
    -- PPO
    -- DPO 
    -- GRPO.
    \item How each PO differ from another. 
    \item GRPO originally introduced for verifiable tasks --- literature on how it was used for non-verifiable tasks. 
    \item Potential limitations for existing GRPO methods on non-verifiable tasks.
    \item How our method address this limitation.
    \textcolor{blue}{(3) should be changed --- robust }
\end{itemize}
\fi

Large language models (LLMs) are trained on a massive collection of diverse datasets; as a result, LLMs acquire a wide variety of goals and skills. When using an LLM for a specific task (e.g., summarization of text), some of its goals and skills are more desirable than others. Therefore, we hope to optimally select a subset of goals and skills of an LLM. Existing methods steer the LLM to align human preferences using reinforcement learning (RL), one of which is GRPO. 

GRPO approximates the actor and critic required for Proximal Policy Optimization (PPO)~\cite{schulman2017proximal} by redefining the advantage as a $z$-score of resampled rewards, making it more amenable for computation. GRPO has been shown to be particularly successful in strengthening the reasoning capabilities of LLMs~\cite{shao2024deepseekmath}, where each response is evaluated as either correct or incorrect~(i.e., binary reward). GRPO has been less explored for non-verifiable tasks that rely on external reward models. In this paper, we adapt the GRPO algorithm for non-verifiable tasks by redefining their $z$-score advantage. Our advantage is a rank-based one, which essentially discards everything but the ordering of the rewards. 

The most straightforward way to elicit desired behaviors from LLMs is supervised fine-tuning (SFT) on high quality human responses. Meanwhile, the most widely used alignment approaches are based on reinforcement learning from human feedback. Here, they first train a reward model to capture pairwise human preferences, and then optimize the LLM as a policy to maximize this learned reward, often using PPO. GRPO was introduced as an alternative to PPO and has shown great success in guiding LLMs to excel at reasoning and mathematical tasks. In such verifiable settings, a learned reward model is often unnecessary because each prompt has a well-defined correct answer.

\begin{figure}
    \centering
    \includegraphics[width=1\linewidth]{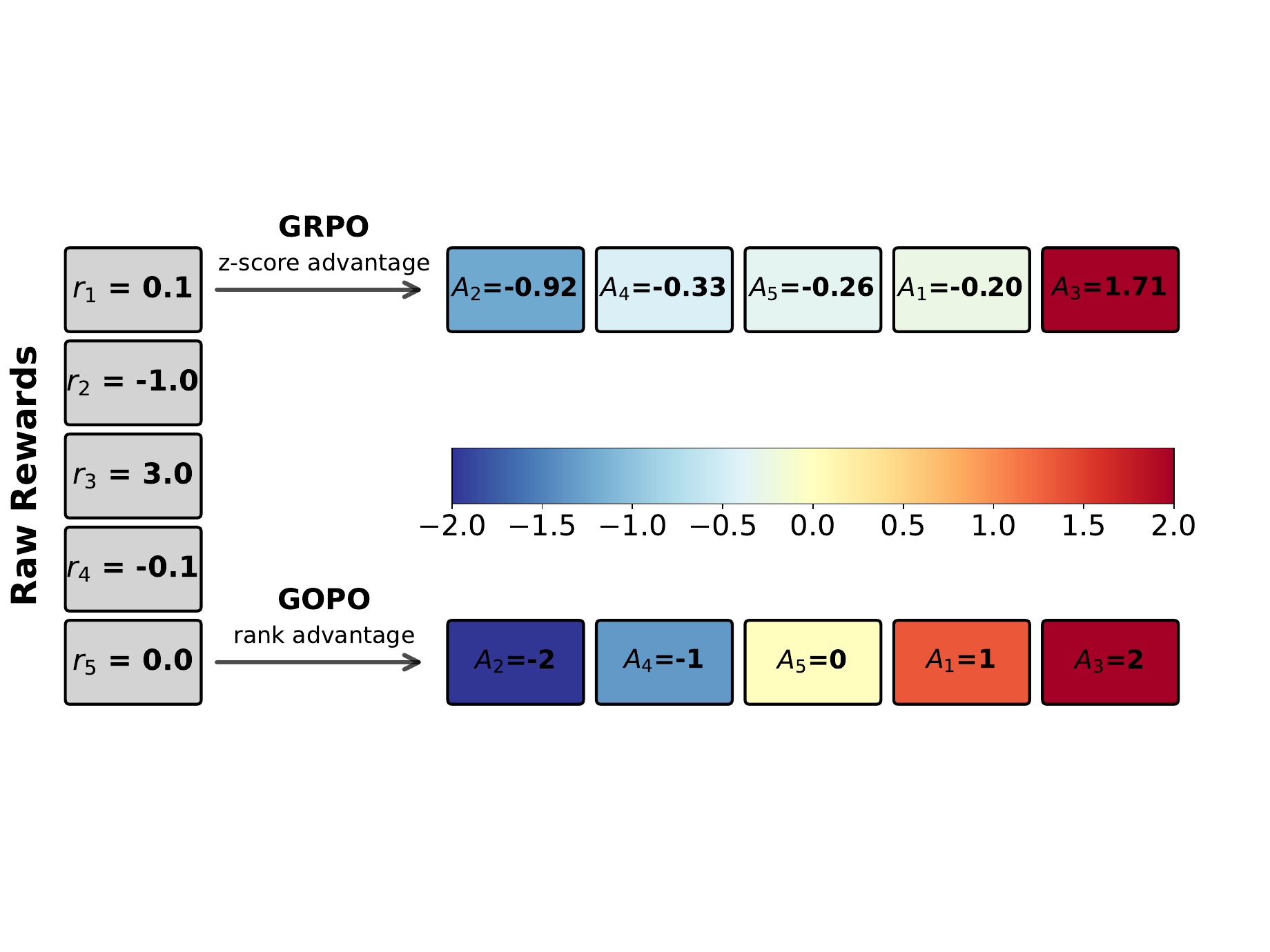}
    \caption{\textbf{GOPO vs. GRPO advantage transformations.} For a fixed prompt with rewards $\{r_i\}$, GRPO uses a $z$-score transformation that centers and scales rewards within the group, while GOPO uses a rank-based transformation that retains only the ordering. $z$-score advantages preserve relative magnitudes among rewards (e.g., similar colors for $A_1, A_4, A_5$ reflect similar raw-reward affinities), whereas rank-based advantages discard scale and can assign different heat levels to rewards with similar magnitudes.}
    \label{fig:std}
\end{figure}

In this paper, we introduce \emph{Group Ordinal Policy Optimization} (GOPO), an alternative to GRPO that removes sensitivity to the noisy and poorly calibrated \emph{magnitude} of reward-model scores by using only their within-prompt \emph{rank order}. This design is motivated by how reward models for non-verifiable tasks are typically trained: Bradley--Terry-style pairwise objectives primarily learn \emph{relative} preferences, so comparisons (which response is better) are often more reliable than absolute score differences. GOPO therefore discards interval-scale information and injects only ordinal information into the RL update, yielding more stable learning and faster improvement. Empirically, GOPO provides consistently stronger guidance than GRPO, improving training reward trajectories and policy quality across training steps (as measured by LLM-as-judge win rates and benchmark evaluations).

\iffalse

In this paper, we show that GRPO policy updates can be improved by removing sensitivity to the noisy, poorly calibrated magnitude of external reward model scores. We improve on the GRPO style update by keeping only the rank order of the rewards. We propose Group Ordinal Policy Optimization (GOPO), an algorithm that optimizes an objective function that implicitly forces the model to focus only on the rank of the responses. 

Reward models output real numbers for non-verifiable tasks and are usually trained using the Bradley--Terry loss, which distills the pairwise order information of two responses of a single prompt. As a result, whether one reward is larger than the other is often more accurate than the absolute distance between the rewards. GOPO therefore discards the possibly noisy and bias-prone interval-scale information, injecting only rank information into the RL update, thereby yielding more stable learning and faster improvement. Across training reward trajectories and policy quality at varying training steps (measured by LLM-as-judge win rates and benchmark evaluations), GOPO provides consistently stronger guidance than GRPO. 
\fi

We summarize our contributions as follows. Across a suite of non-verifiable tasks and base model sizes:
\begin{enumerate}
    \item We show that GOPO-updated policies consistently achieve higher training reward trajectories than those trained with GRPO. 
    \item We show that GOPO attains superior test performance, measured by benchmark scores and/or win rates judged by frontier large, general-purpose language models. 
    \item We show that GOPO is more sample-efficient, i.e., it reaches comparable output quality earlier in training.
    % \item Through extensive experiments, we show that rank is all we need for a policy update and that all other reward information are nuisances. 
\end{enumerate}

% \begin{figure}
%     \centering
%     \includegraphics[width=0.8\linewidth]{figure/seed42_lines_image.pdf}
%     \caption{Caption}
%     \label{fig:placeholder}
% \end{figure}

\begin{figure*}[!ht]
    \centering
    \begin{tabular}{ccc}
        \includegraphics[width=0.32\linewidth]{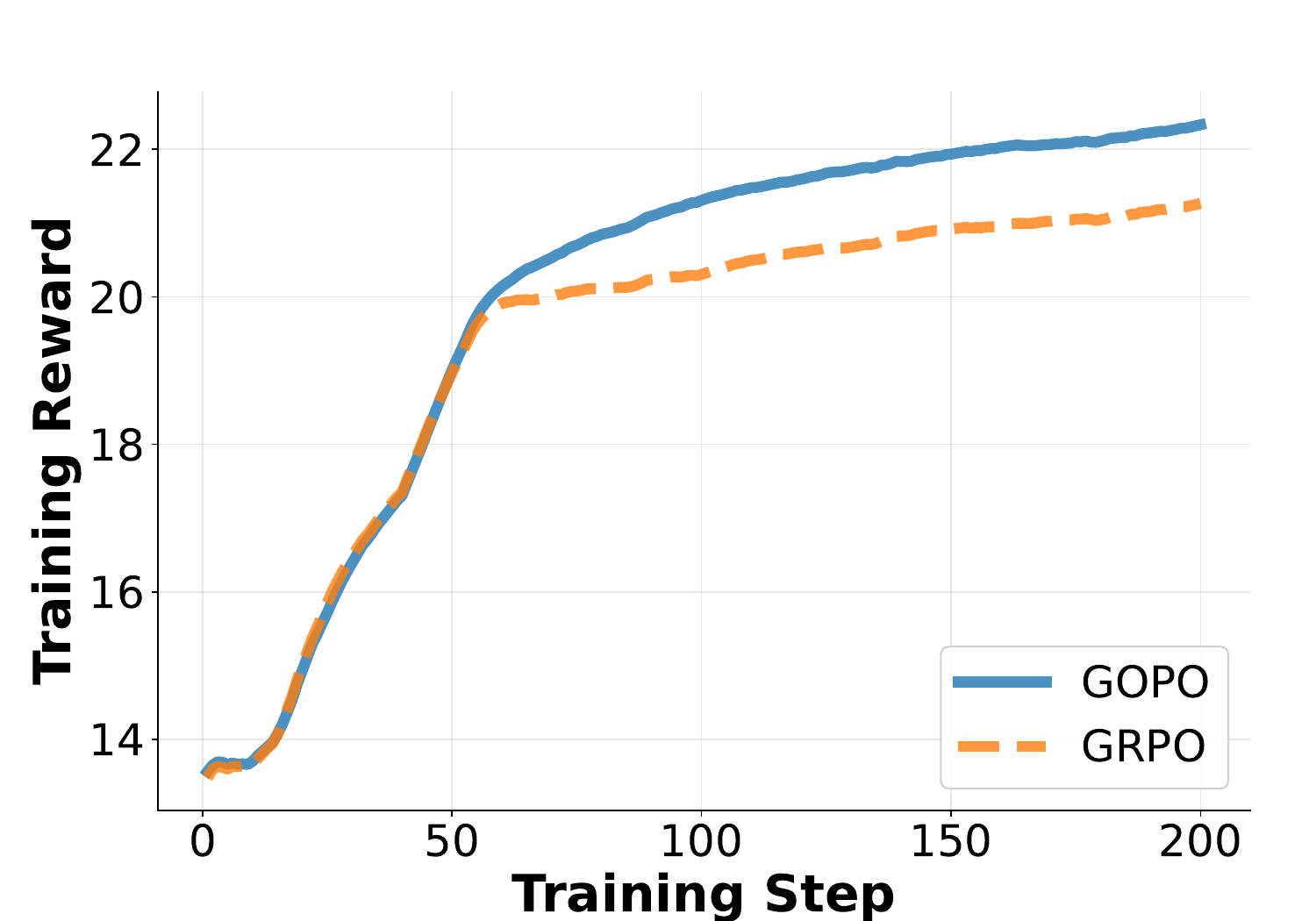} &
        \includegraphics[width=0.32\linewidth]{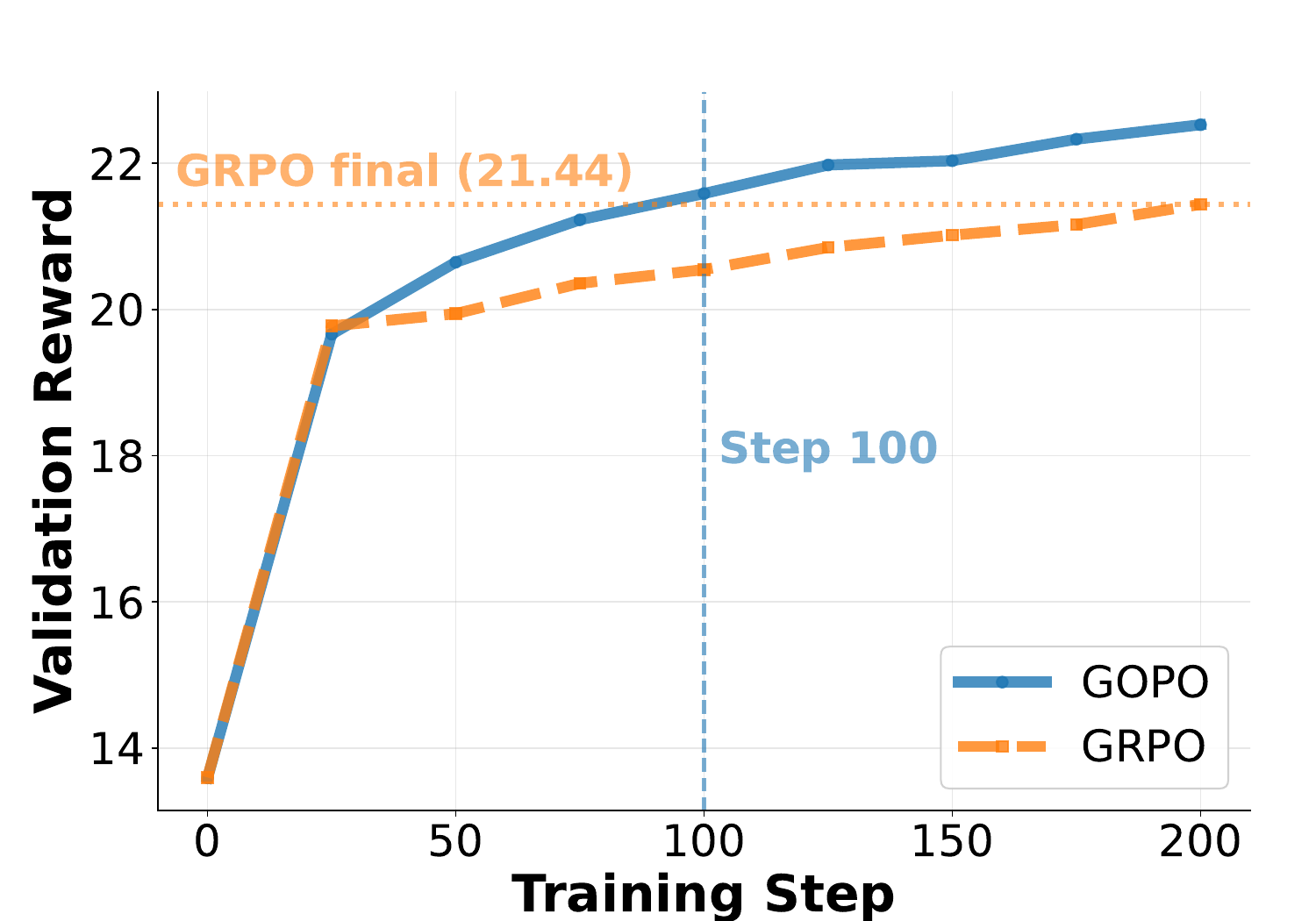} &
        \includegraphics[width=0.32\linewidth]{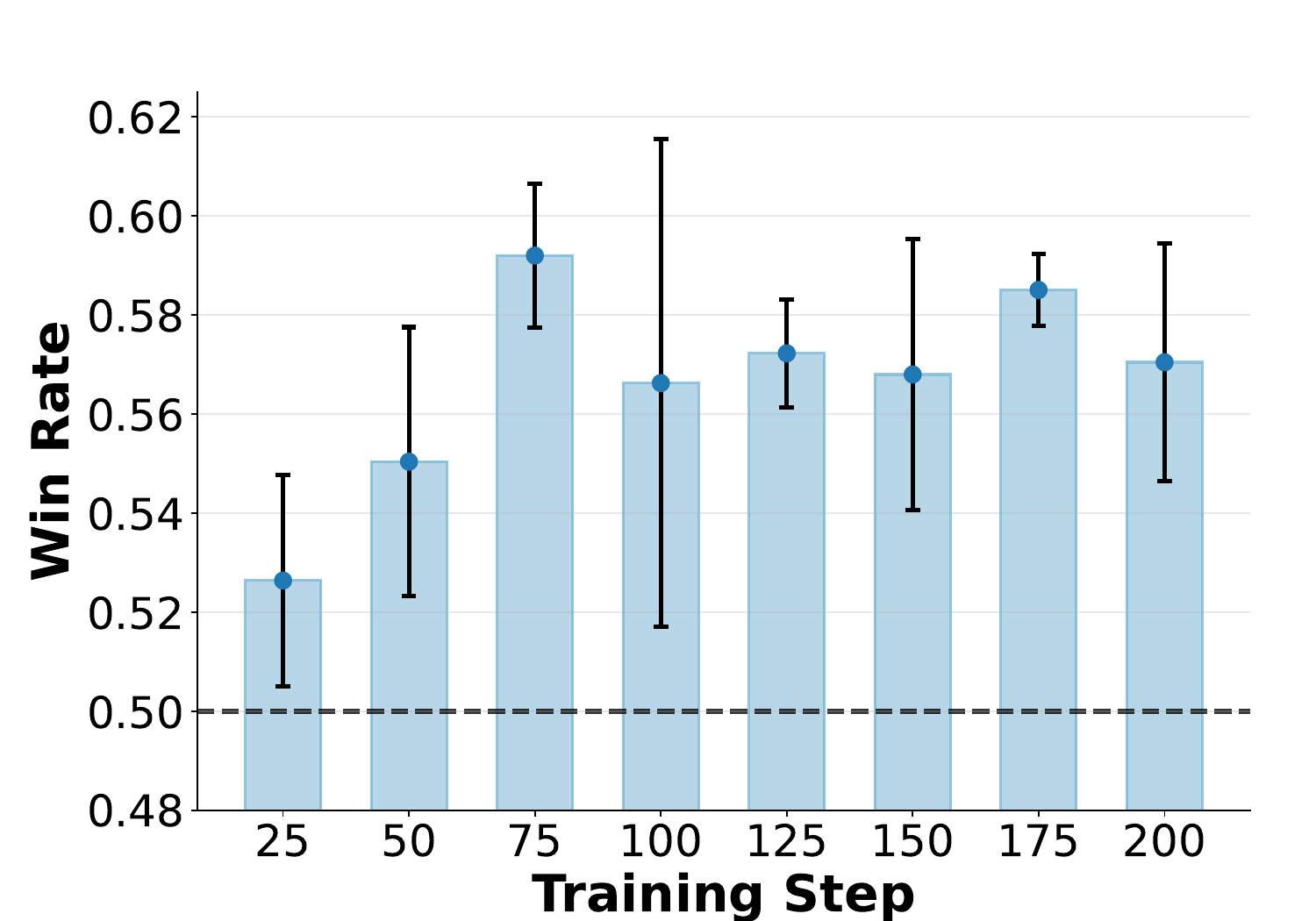} \\
        (a) Training Reward &
        (b) Validation Reward &
        (c) LLM-as-judge Evaluation
    \end{tabular}
    \caption{\textbf{Base model: \texttt{Qwen3-8B}, Reward model: Skywork (\texttt{Qwen3-8B}), Task: \texttt{TLDR}.} 
    % \rrd{this is a good one to push to main paper.} 
    % \rrd{the last point on the 3rd plot here might confuse folks - so what if we stop at 200 steps. + y-axis needs to be limited}
    Figures (a) and (b) plot the per-training step policy's generation mean reward using prompts in the training dataset and validation dataset respectively---both rewards are consistently higher for GOPO updated policies throughout training. Figure (c) reports the LLM-as-judge win-rate (see Section \ref{sec:eval} on how the win-rate is defined) of GOPO updated policies against GRPO updated policies at matched training steps---for \emph{multi-seed generations}, GOPO consistently improves the win-rates throughout all training steps. The policy generation temperature for Figure (c) is fixed at $0.5$; see Table \ref{tab:qwen8b_combined} in Section \ref{sec:tldrchat} for win-rates on varying temperatures. Lastly, validation reward of GRPO at its last training step is achieved earlier for GOPO (step $100$), and the GOPO win-rate at its earlier training step against the final GRPO is $0.52$.
    }
    \label{fig:tldr-8b}
\end{figure*}

\iffalse
\begin{figure*}[!ht]
    \centering
    \begin{tabular}{cc}
        \includegraphics[width=0.45\linewidth]{figure/tr_ultrachat_qwen4b_seed42.pdf} &
        \includegraphics[width=0.45\linewidth]{figure/chat_multi_model_best_worst_survey.pdf}\\
        (a) Training Reward Trajectories & (b) Test Evaluation (LLM-as-Judge)
    \end{tabular}
    \caption{\textbf{Train and test evaluations (\texttt{UltraChat}})}
    \label{fig:chat-train-test}
\end{figure*}
\fi

\section{Related Work}

This work lies at the intersection of reinforcement learning from human feedback, preference-based policy optimization, and recent efforts to improve the stability and efficiency of post-training for large language models. We focus in particular on how reward signals are shaped during policy optimization, and how this choice affects policy optimization for training language models on non-verifiable tasks.

\paragraph{RLHF and Policy Optimization for Language Models}

Modern LLM alignment pipelines typically consist of a combination of SFT and RL using a learned reward model trained from human preference data \citep{ouyang2022training, stiennon2020learning, ziegler2019finetuning}. The dominant optimization method at this stage has been PPO, often augmented with a KL penalty to prevent excessive drift from the reference SFT model. 

To reduce the variance and engineering complexity of token-level value estimation, several works replace explicit critic learning with {group-based} baselines. GRPO \citep{shao2024deepseekmath} is a prominent example in which multiple completions per prompt are sampled and their rewards are standardized within the group to form advantages, and it has been especially successful in {verifiable} domains (e.g., math and reasoning) where rewards are often binary and relatively well-calibrated.

However, in many alignment tasks—summarization, instruction following, and open-ended dialogue—rewards are non-verifiable and come from external reward models. In such settings, reward magnitudes are known to be noisy, poorly calibrated, and sensitive to distribution shift. Our work builds directly on GRPO but questions whether its reliance on cardinal reward information (via $z$-scoring) is appropriate in this regime.

\paragraph{Preference Learning and Ordinal Information}

Reward models in RLHF are typically trained using pairwise preference data under a Bradley--Terry or similar logistic ranking formulation \citep{bradley1952rank, christiano2017deep}. Such models are fundamentally optimized to capture {relative orderings} between responses rather than absolute reward scales. As a result, while the sign of reward differences may be reliable (i.e., which response is better), the magnitude of those differences is often biased.

This mismatch between {ordinal supervision} and {cardinal policy optimization} has been recognized in prior works related to preference-based reinforcement learning (PbRL)~\citep{busa2014preference, jain2013learning, sadigh2017active}. Most PbRL approaches first reconstruct a latent reward function and then optimize it, thereby re-introducing scale sensitivity. In contrast, our method directly incorporates ordinal structure into the policy gradient itself, bypassing the need to trust reward magnitudes.

Related ideas also appear in contextual dueling bandits \citep{yue2012dueling, dudik2015contextual}, where policies are evaluated through pairwise comparisons rather than absolute payoffs. While those works focus on online preference queries and regret minimization, the underlying principle—that {rank information can be sufficient for policy improvement}—closely aligns with our approach.

\paragraph{Variance Reduction, Robustness, and Advantage Design}

The design of the advantage function is central to policy-gradient stability. PPO and GRPO both rely on normalization (e.g., $z$-scoring) to control gradient scale. However, standardized rewards can still be sensitive to outliers and reward model miscalibration, especially as the number of sampled completions per prompt grows.

Several recent works have explored alternative ways to control variance and training dynamics in LLM post-training. For example, curriculum or variance-aware sampling strategies prioritize prompts with informative reward variation \citep{jiang2025vcrl}, while budgeted or knapsack-style RL methods allocate exploration resources adaptively across prompts \citep{li2025knapsack}. These methods aim to improve {which data points} are emphasized during training. In contrast, GOPO modifies {how reward information is encoded} within each prompt group: rather than filtering prompts with small variance, we amplify even small but reliable ordinal differences by mapping them to fixed, evenly spaced ranks.

% Our approach is also related to broader themes in robust reinforcement learning, where reducing sensitivity to reward scale and heavy-tailed noise can improve stability. By construction, rank-based advantages are uniformly bounded and invariant to monotone transformations of rewards, providing a form of robustness that standardized advantages lack.

\paragraph{Multi-Stage Post-Training}

Large-scale LLM alignment pipelines increasingly involve multiple stages of RL, often starting with verifiable tasks and then moving to preference-based, non-verifiable objectives. For example, GRPO has been applied after earlier stages such as reinforcement learning with verifiable rewards (RLVR) or reasoning-focused training \citep{shao2024deepseekmath}. While these works demonstrate that GRPO can be reused in later stages, they largely retain the same $z$-score advantage design without re-examining its suitability when rewards come from learned preference models. GOPO can be viewed as a replacement for GRPO in second-stage RL, specifically tailored to the statistical properties of reward models trained from pairwise preferences. 

\begin{figure*}[!ht]
    \centering
    \begin{tabular}{ccc}
        \includegraphics[width=0.32\linewidth]{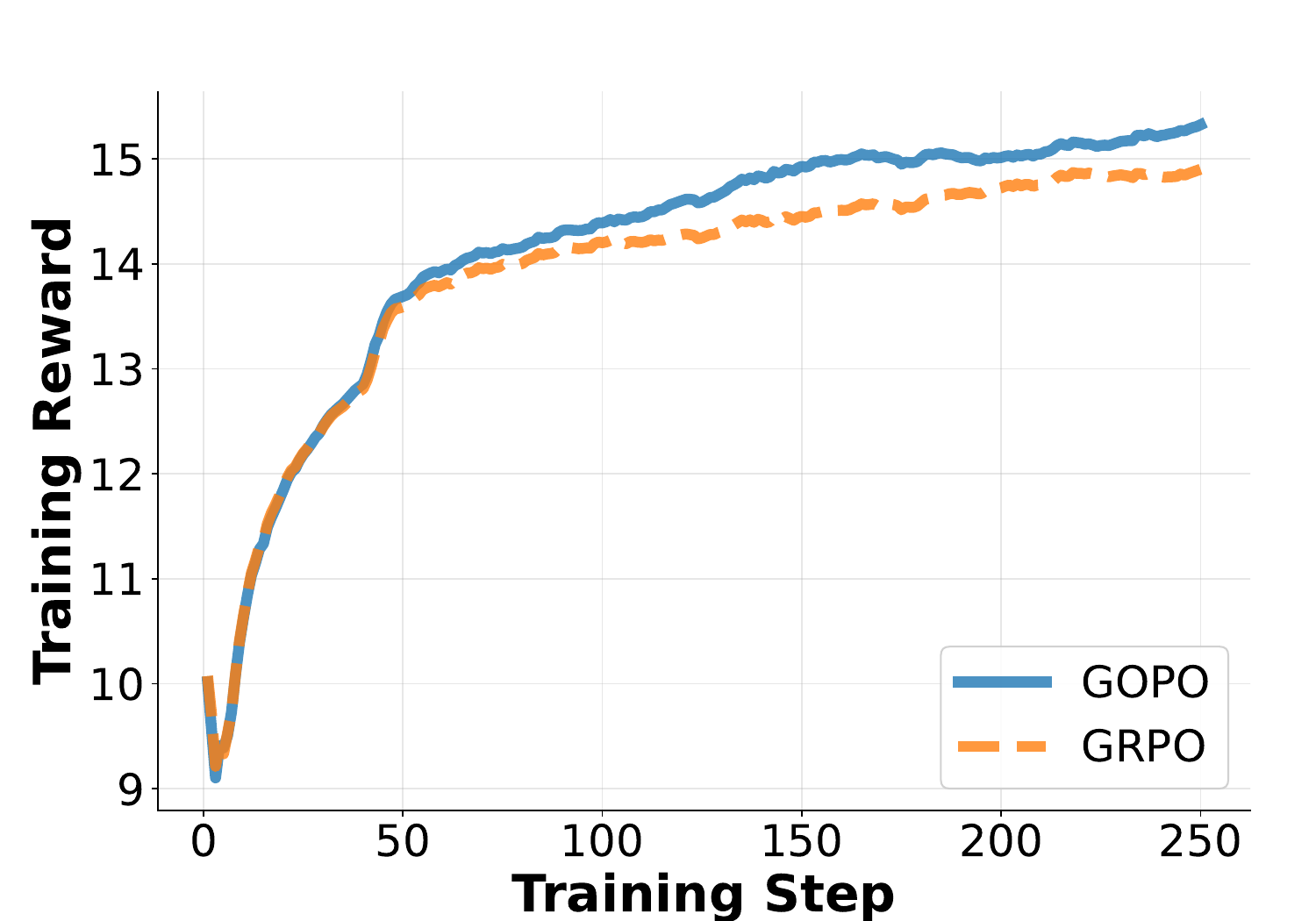} &
        \includegraphics[width=0.32\linewidth]{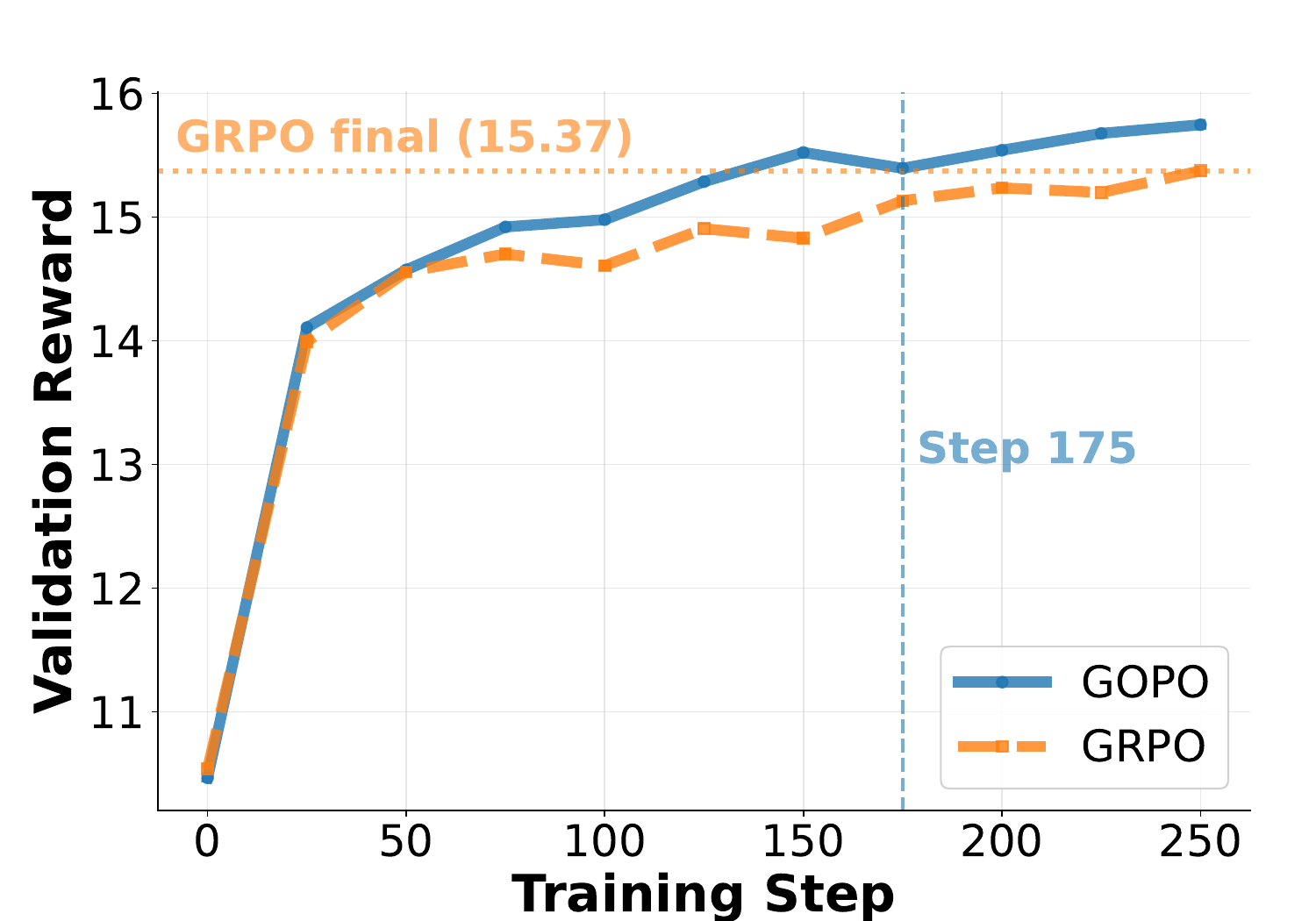} &
        \includegraphics[width=0.32\linewidth]{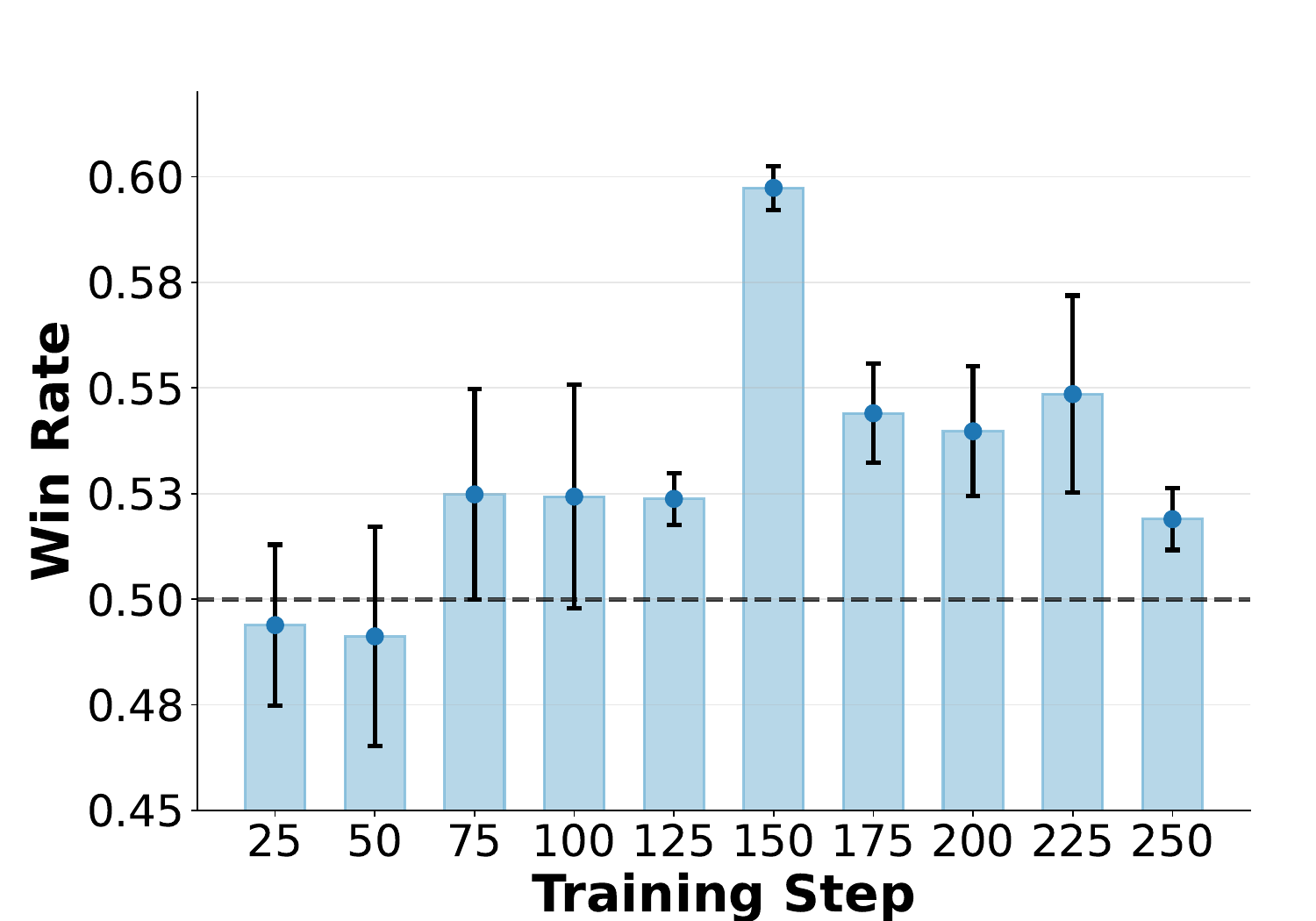} \\
        (a) Training Reward &
        (b) Validation Reward &
        (c) LLM-as-judge Evaluation
    \end{tabular}
    \caption{\textbf{Base model: \texttt{Qwen3-4B}, Reward model: Skywork (\texttt{Qwen3-8B}), Task: \texttt{UltraChat}.} Figure (a) and (b) are the per-training step policy's generation mean reward using prompts in the training dataset and validation dataset respectively---both rewards are consistently higher for GOPO updated policies throughout training. Figure (c) contains the LLM-as-judge win-rate (see section \ref{sec:eval} on how the win-rate is defined) of GOPO updated policies against GRPO updated policies at their identical training steps---for \emph{multi-seed generations}, GOPO consistently improves the win-rates throughout most of the training steps. The policy generation temperature for Figure (c) is fixed at $0.5$; see Table \ref{tab:qwen_small_combined} in Appendix \ref{app:robust} for results on varying temperatures. Lastly, validation reward of GRPO at its last training step is achieved earlier for GOPO (step $175$), and the GOPO win-rate at its earlier training step against the final GRPO is $0.52$.}
    \label{fig:chat-4b}
\end{figure*}

\section{Method}

In what follows, we formally introduce GOPO in the context of reinforcement learning with ranking information in the non-verifiable reward setting. Before doing so, we provide a brief overview of how GRPO operates. We then conclude with intuition for why, in this setting, rank-based advantages are preferable to $z$-scores, along with theoretical insights into the robustness of ranking-based training dynamics.

\subsection{Review of GRPO}

\newcommand{\objfunc}{\mc J(\theta)}

In this section, we review GRPO~\cite{shao2024deepseekmath}. Denote $\pi_\theta$ as a generic parametrized policy, which here is the language model itself. The baseline policy model, denoted $\pizero$, is typically a policy refined through SFT. Policy optimization is an iterative online optimization process: the previous policy $\piold$ generates outputs that are used to construct an objective function $\objfunc$, and $\pinew:=\argmin_\theta\objfunc$. In the next training round, $\pinew$ is set as $\piold$. Specifically, for GRPO, the following sequence of recipes is needed to construct the objective function from the previous policy:
% \textcolor{red}{talk about how current estimate as well is used to find a next estimate using the objective, i.e., we should be providing the acutal algorithm....}. 
% The objective function of GRPO is constructed with the following sequence of recipes:
\begin{enumerate}
    \item Prompts are sampled $q \sim P_Q$ from some distribution $P_Q$ over the set of prompts in the training data.\footnote{Notice that we omit any subscript indexing with regard to different prompts in the training dataset for notational convenience.  }
    \item For each prompt $q$, completions are generated $o_i = [o_{i, 1}, ..., o_{i, T}]$ (total of $T$ tokens) for $i = 1, ..., G$ from the old policy $\piold$.
    \item A reward model $r_\phi$ assigns reward $r_{i}$ for each prompt and completion pair $(q, o_i)$.\footnote{Originally, GRPO was designed for math, logic tasks that entailed verifiable rewards, meaning that $r_i \in \{0, 1\}$.}
\end{enumerate}

To formalize the GRPO objective $\mc J(\theta)$, let $\pi_\theta$ denote the current policy and
$\pi_{\mathrm{old}}$ the reference (behavior) policy. For output $i$ at step $t$, define the
likelihood ratio
\[
\pi_t(\theta) := \frac{\pi_\theta\left(o_{i,t}\mid q, o_{1:t-1}\right)}{\pi_{\mathrm{old}}\left(o_{i,t}\mid q, o_{1:t-1}\right)},
\]
where $o_{1:s}=[o_1,\ldots,o_s]$ denotes the token prefix. Next define the clipped variant
\[
f(\hat A_{i,t},\pi_t(\theta)) :=
\min\Big\{\pi_t(\theta)\hat A_{i,t},\mathrm{clip}\big(\pi_t(\theta),1\pm\varepsilon\big)\hat A_{i,t}\Big\}.
\]
With these definitions, the GRPO objective is given by $\mc J(\theta)$ (see \cite{shao2024deepseekmath}) which is

\begin{equation}\label{eq:obj-func}
    \begin{aligned}
        \mc J(\theta) &= \E_{q \sim P_Q}\bigg[ \frac{1}{G}\sum_{i = 1}^{G} \frac{1}{T}\sum_{t = 1}^{T} f(\hat A_{i, t}, \pi_t(\theta)) \bigg]\\
        &- \beta \mathrm{KL}( \pi_\theta \| \pizero ).
    \end{aligned}
\end{equation}
The advantage $\hat A_{i, t}$ reflects the relative importance of a response $o_i$ across the completions for a given prompt. For a fixed prompt $q$, the advantages of GRPO \cite{shao2024deepseekmath} are set as the standardized rewards \textit{broadcasted across the token index $t$}; i.e., for all $i = 1, ..., G, t = 1, ..., T$, $\hat A_{i, t} = \hat A^{\mrm{std}}_{i, t}$ where
\begin{equation}\label{eq:std-advantage}
\begin{aligned}
    \hat A^{\mrm{std}}_{i, t} := \frac{r_i - \mrm{mean}(r_1, ..., r_G)}{\mrm{std}(r_1, ..., r_G)}.
\end{aligned}
\end{equation}
% So a single training step updates the old policy to a new one $\pinew = \pi_{\theta^\star}$where $\theta^\star = \argmax_{\theta}\mc J(\theta)$, and the next training step iterates this by resetting $\piold = \pinew$.

% \textcolor{red}{give the update, talk about how we terminate etc or maybe this can be generally written after the separate advantage computation for each method}

\subsection{Group Ordinal Policy Optimization (GOPO)}

We innovate the GRPO algorithm by redefining the advantage $\hat A_{i, t}$ within the objective function $\objfunc$ in (\ref{eq:obj-func}).
% The objective $\mc J_{\mrm{GRPO}}$ is designed to guide policies so that it outputs completions that are favored by the reward model.
For post-training a language model on non-verifiable tasks, we claim that disregarding everything except the per-prompt order of rewards improves the original GRPO update \cite{shao2024deepseekmath}.

The advantage defined in (\ref{eq:std-advantage}) encodes the \textit{cardinal} information of rewards, and we discard them to propose 
% and is therefore more nuanced than the \textit{ordinal} information of rewards. Recall that the interval between two rewards is expressed as $r_i - r_j = \mrm{sign}(r_i - r_j)\cdot|r_i - r_j|$ where the quantity $\mrm{sign}(r_i - r_j)$ is positive if and only if $r_i > r_j$. 
propose a rank-transform advantage $\hat A^{\mrm{rank}}_{i,t}$, which is assigned to the completions $o_1, ..., o_G$ as follows:
\begin{figure*}[!ht]
    \centering
    \begin{tabular}{ccc}
        \includegraphics[width=0.32\linewidth]{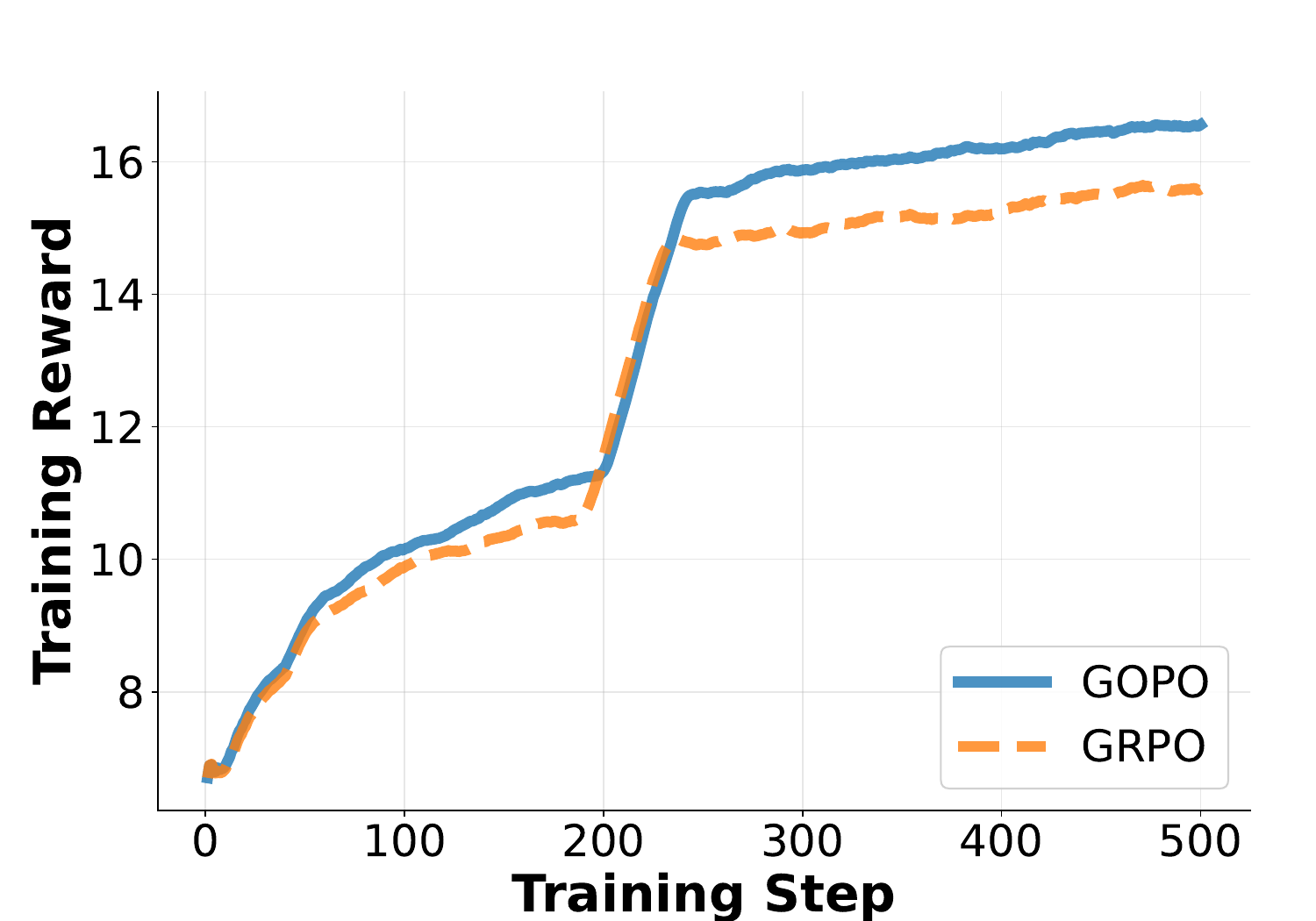} &
        \includegraphics[width=0.32\linewidth]{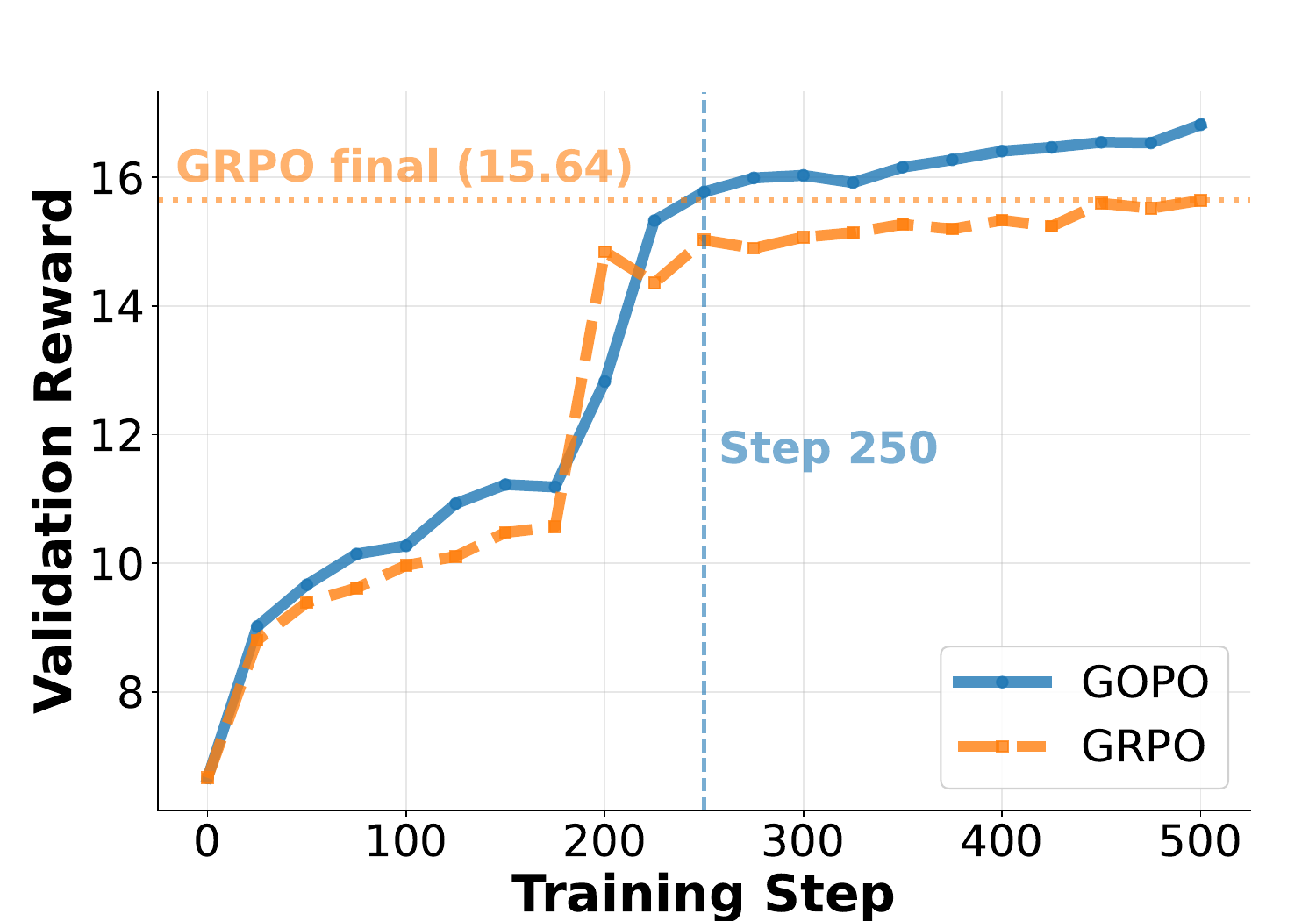} &
        \includegraphics[width=0.32\linewidth]{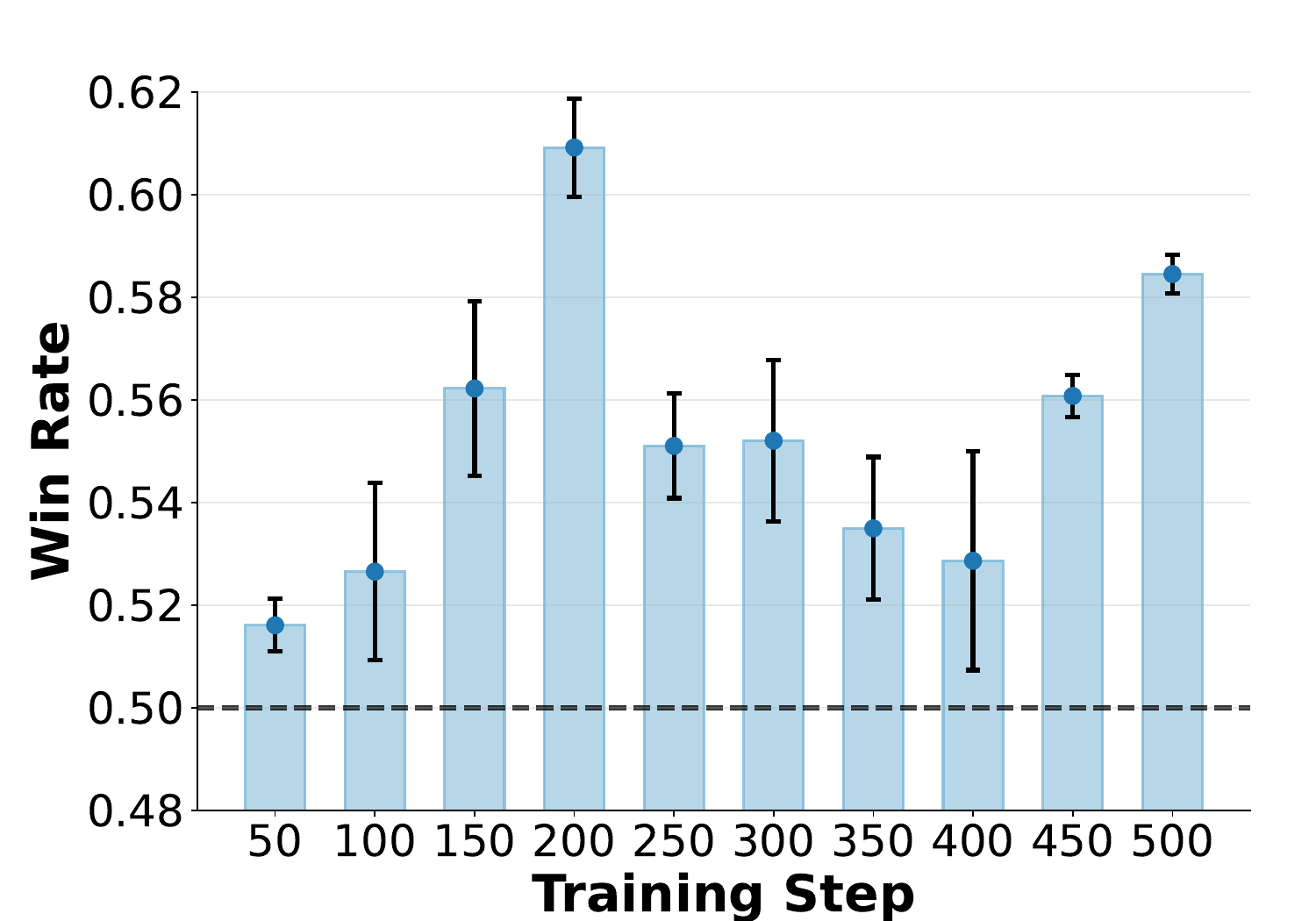} \\
        (a) Training Reward &
        (b) Validation Reward &
        (c) LLM-as-judge Evaluation
    \end{tabular}
    \caption{\textbf{Base model: \texttt{Qwen3-1.7B}, Reward model: Skywork (\texttt{Qwen3-8B}), Task: \texttt{TLDR}.} 
    % \rrd{you can smoothen the first plot for all such training plots; for the middle plot just put a label on top of the horizontal dashed line of GRPO final and another lable for the numnber of steps on the vertical line; make the two dashed line thicker and may be blue horizontal and red vertical with some alpha <1; for the right plot start the y-axis at 0.45 and end at 0.65} 
    Figure (a) and (b) are the per-training step policy's generation mean reward using prompts in the training dataset and validation dataset respectively. Both rewards are consistently higher for GOPO updated policies throughout training. Figure (c) contains the LLM-as-judge win-rate (see Section \ref{sec:eval} for how win-rate is defined) of GOPO updated policies against GRPO updated policies at their identical training steps---for \emph{multi-seed generations}, GOPO consistently improves the win-rates throughout all training steps. The policy generation temperature for Figure (c) is fixed at $0.5$; see Table \ref{tab:qwen_small_combined} in Appendix \ref{app:robust} for results on varying temperatures. Lastly, validation reward of GRPO at its last training step is achieved earlier for GOPO (step $250$), and the GOPO win-rate at its earlier training step against the final GRPO is $0.52$.
    % ; refer to Figure \ref{fig:tldr-kl} for the win-rates using other sampling temperatures.
    }
    \label{fig:tldr-1.7b}
\end{figure*}
\begin{enumerate}
    \item Completion $o_i$ with \textit{highest} reward $r_i$ is assigned advantage $\hat A_{i, t}^{\mrm{rank}} = 2$ for all tokens. 
    \item The completion $o_j$ with \textit{lowest} reward $r_j$ is assigned advantage $\hat A_{j, t}^{\mrm{rank}} = -2$ for all tokens. 
    \item All the remaining completions' advantages are placed between $(-2, 2)$ in an equidistant manner according to their reward ranks. 
\end{enumerate}

A more compact definition of rank-transform advantage $\hat A_{i, t}^{\mrm{rank}}$ is as follows:
\begin{equation}\label{eq:rank-advantage}
    \begin{aligned}
        \hat A_{i, t}^{\mrm{rank}} &:= 2 - \{\rho(i) - 1\}\cdot\frac{4}{G - 1}\\
        &\text{for all $i = 1, ..., G$   \ and \ $t = 1, ..., T$,}
    \end{aligned}
\end{equation}
where $\rho: [G] \to [G]$ is the rank mapping of each $i$th completion $o_i$ according to its reward $r_i = r_\phi(q, o_i)$, i.e., $i$th completion with highest rank receives rank $1$ and the lowest receives rank $G$. The upper and lower bounds of $\hat A_{i, t}^{\mrm{rank}}$ align with $2$ standard deviations of the standardized variables and the advantages are \textit{equi-distant} by the amount $\frac{4}{G - 1}$.

\subsection{Why rank?}\label{sec:why-rank}

% \textcolor{blue}{TODO: Chitchat on the intuition - DS done.}

Reward models are trained on preference data \cite{bradley1952rank}, so they are particularly effective at determining when one completion is better than another, rather than assessing the absolute quality of a response. The latter is widely recognized as a much harder problem—one that reward models are not well suited to solving.

However, current policy optimization methods typically rely on absolute rewards. In the context of group-based policy optimization for non-verifiable rewards, we identified an opportunity to instead leverage the information that reward models excel at providing relative comparisons. This yields a natural ordering of completions. We believe that aligning reward models and policy models in this way leads to improved performance and more robust training procedures, characterized by faster convergence and greater stability.

% \textcolor{blue}{KC: Also the rank of true rewards are preserved if the approximation bias of parametrized reward models is small. This explains the utility of GOPO when using reward models that are not necessarily trained on preference data. Let's find some literature for this.}

\paragraph{Gradient norms}

We examine the behavior of the gradient norm of policy updates when the per-prompt sample size $G$ is small. We argue that for small $G$ (e.g., $G<10$), the norm of the gradients under the GOPO update generally has higher variance. This is related to the fact that the uniform distribution over rank-transformed advantages attains maximal entropy. 

We simplify the setting by considering a single prompt and a smoothed version (i.e., disregard non-differentiable components~ $\min$ and $\mrm{clip}$) of the objective function (\ref{eq:obj-func}). Let $q \sim P_Q$ be a randomly sampled prompt, then the smooth objective is 
\begin{equation*}
    \begin{aligned}
        \mc J(\theta) &= \frac{1}{G} \sum_{i = 1}^{G} \frac{1}{T}\sum_{t = 1}^{T} \pi_t(\theta) \hat A_{i, t} - \beta  \mrm{KL}(\pi_\theta \| \pi_{\mrm{ref}})\\
        &=: \mc J_{1}(\theta) - \beta  \mrm{KL}(\pi_\theta \| \pi_{\mrm{ref}})
    \end{aligned}
\end{equation*}
where $\mc J_{1}(\theta)$ refers to the double sum in the above display. Define the random vector $X_i :=T^{-1}\sum_{t = 1}^{T} \nabla_\theta \log \pi_t(\theta) \pi_t(\theta)$ (the index $i$ is implicit in $\pi_t(\theta)$, see (\ref{eq:obj-func})). Note that we drop the index $b$ for the advantage $\hat A_{i, t}$ and the vector $X_i$ as we are in the single prompt setting. As advantages are broadcasted throughout tokens (set $\hat A_{i, t} = \hat A_i$), we observe
\begin{equation*}
    \begin{aligned}
        \nabla_\theta \mc J (\theta) &= \nabla_\theta \mc J_1(\theta) - \beta \nabla_\theta \mrm{KL}(\pi_\theta\|\pi_{\mrm{ref}})\\
        &= \frac{1}{G}\sum_{i = 1}^{G}\hat A_{i} X_i - \beta \nabla_\theta \mrm{KL}(\pi_\theta \| \pi_{\mrm{ref}}).
    \end{aligned}
\end{equation*}
% We consider the gradient norm of the first sum of the above display. Define the random vector $X_i^{(b)} :=T^{-1}\sum_{t = 1}^{T} \nabla \log \pi_t(\theta) \pi_t(\theta)$; note that the index $i$ and $b$ are implicit in $\pi_t(\theta)$, see (\ref{eq:obj-func}).  

% For each completion $i \in \{1,\dots,n\}$ and token $t \in \{1,\dots,T\}$,
% let $o_{i,t}$ be the $t$-th token in the $i$-th completion.
% Define its score gradient
% \[
% g_{i,t} := \nabla_\theta \log \pi_\theta(o_{i,t}\mid o_{i,<t}) \in \mathbb{R}^d.
% \]
% The per-completion gradient is
% \[
% X_i := \sum_{t=1}^T g_{i,t},
% \]
% and the weighted GRPO gradient estimator is
% \[
% \widehat g(w) := \frac{1}{n}\sum_{i=1}^n w_i X_i,
% \]
% with weights $w_i$ being either rank-based $\hat A_{i, t}^{\mrm{rank}} \in [-2,2]$ or standardized rewards $\hat A_{i,t}^{\mrm{std}}$.

\begin{theorem}[Larger Gradient Norms]\label{thm:B1-gopo-inflation}
Let $\mc F = \sigma(q, A_1, \dots, A_G)$ be the conditioning event and define the centered vectors $\xi := g - \E[g\mid \mc F]$ and $\widetilde{X}_i := X_i - \E[X_i\mid \mc F]$ where $g = \nabla_\theta \mc J_1(\theta)$ is the gradient of the (non-penalized) objective function.
% denote \[g:=G^{-1}\sum_{i=1}^G A_i X_i \qquad \xi := g - \E[g\mid \mc F].\] 
Assume $\widetilde{X}_i$ are conditionally uncorrelated and have second moment $\sigma^2_X < \infty$ almost surely.\footnote{Formally, $\E[\langle \widetilde X_i,\widetilde X_j\rangle\mid\mc F]=0$ for $i\neq j$ and
$\E[\|\widetilde X_i\|^2\mid\mc F]=\sigma_X^2$ for all $i \in [G]$.} Then
\[
\E\|\xi\|^2=\frac{1}{G}\,\E\!\left[\sigma_X^2\cdot \frac{1}{G}\sum_{i=1}^G \hat A_i^2\right].
\]
\end{theorem}
Theorem \ref{thm:B1-gopo-inflation} implies that the variance of the advantages drives the (conditionally centered) gradient norm of the update. Defining $\xi_{\text{GOPO}}$ and $\xi_{\text{GRPO}}$ as the variants of $\xi$ with their specific advantage definitions respectively, we observe
\[
\frac{\E\|\xi_{\text{GOPO}}\|^2}{\E\|\xi_{\text{GRPO}}\|^2} = \frac{4(G+1)}{3(G-1)} > 1.
\]
Note that the gradient norm inflation is particularly big for smaller per-prompt sample size $G$.

\begin{remark}[Connection to KL, Theorem \ref{thm:B1-gopo-inflation}]\label{rem:kl} Policy optimization using either GOPO or GRPO takes the form $\theta_1 \leftarrow \theta_0 + \eta g$, where $g$ is the gradient of the objective function. Such an update of the policy parameter $\theta$ yields the approximation $\mrm{KL}(\pi_{\theta_{\mrm{new}}}\| \pi_{\theta_{\mrm{old}}}) \approx \frac{\eta^2}{2}g^T F(\theta_{\mrm{old}})g$ for the Fisher information matrix $F(\cdot)$. This information encodes the sensitivity of $\mrm{KL}$ to the change of $\theta$; see \cite{schulman2015trust} for details. Given that the step-wise policy updates are local (by design of imposing KL penalization), the Fisher information varies smoothly, allowing us to argue $g^T g$ is proportional to $g^T F(\theta_{\mrm{old}}) g$ up to the Fisher metric. Therefore, larger gradient norms of GOPO naturally yield steep KL increase, see Figure \ref{fig:kl} in Appendix \ref{app:kl}.
\end{remark}

% \textcolor{blue}{@KS we should mention why the fisher information is prop to $E\xi^2$ right?}

In Appendix \ref{app:thm-large-G}, we show that gradient norm of GOPO updates are asymptotically bounded as $G$ grows, implying robustness of GOPO updates for large per-prompt sample size $G$ setting, whereas that of GRPO tends to grow with $G$.

\section{Experimental Setup}

\begin{figure*}[!ht]
    \centering
    \begin{tabular}{ccc}
        \includegraphics[width=0.32\linewidth]{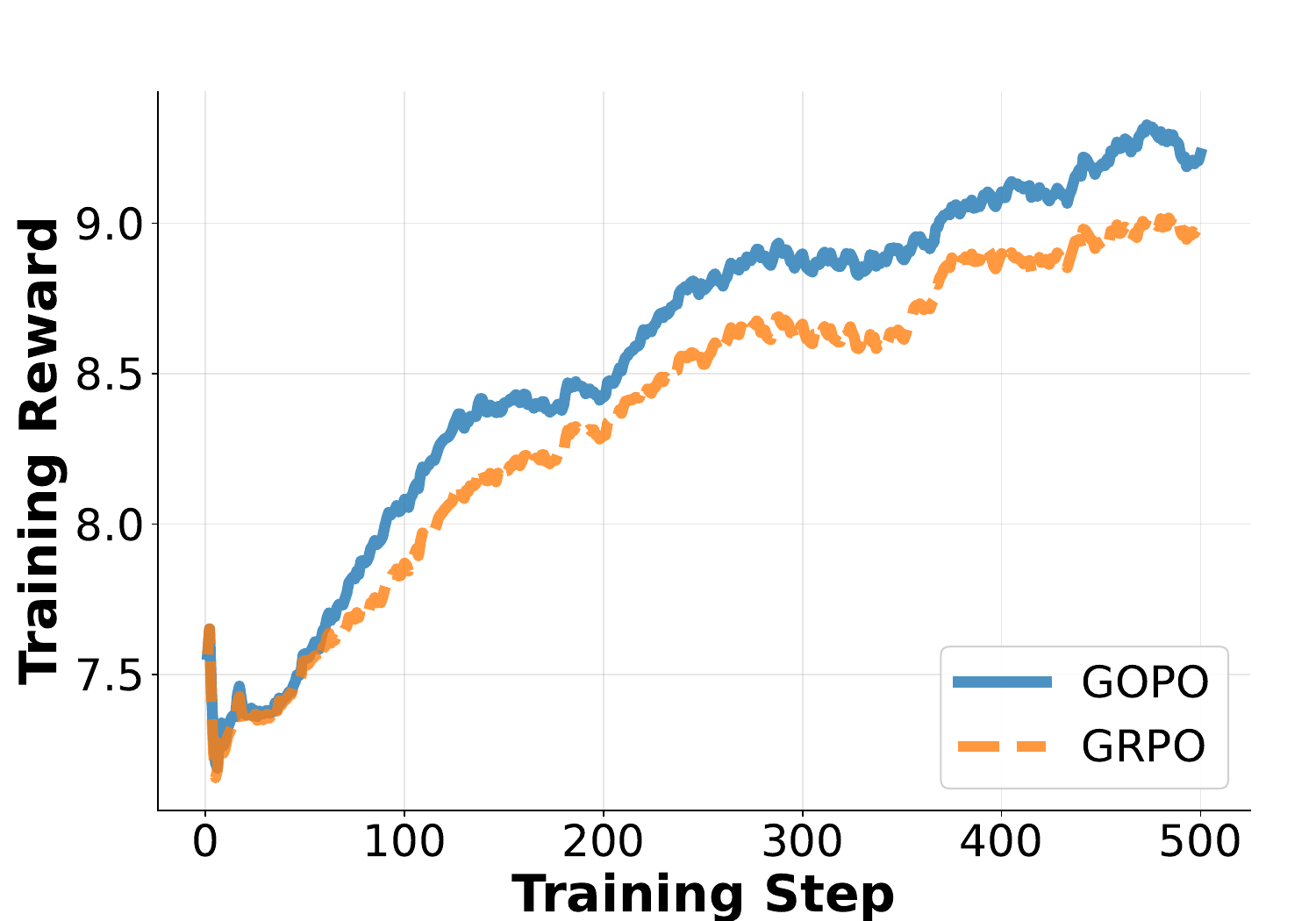} &
        \includegraphics[width=0.32\linewidth]{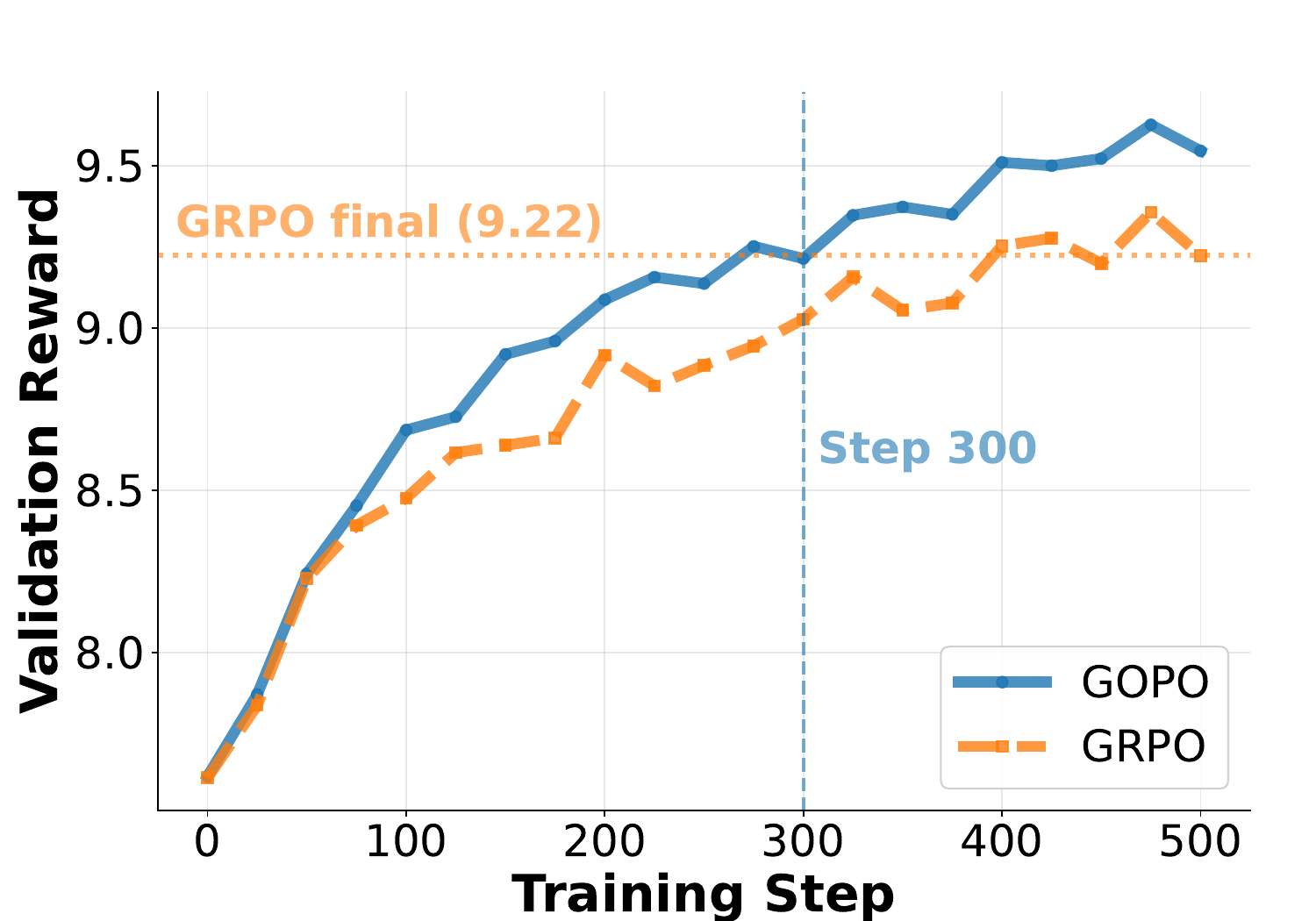} &
        \includegraphics[width=0.32\linewidth]{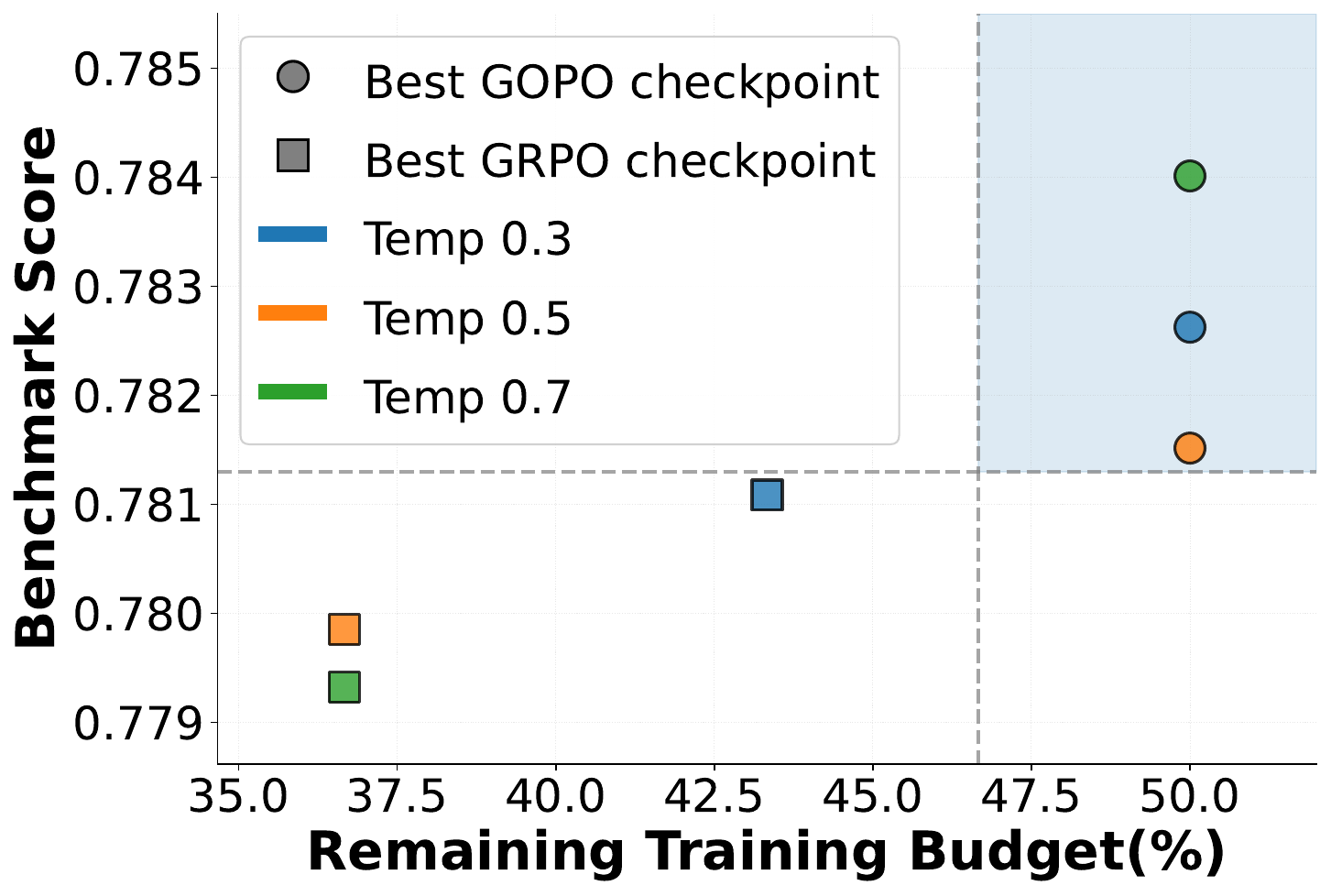} \\
        (a) Training Reward &
        (b) Validation Reward &
        (c) Benchmark score
    \end{tabular}
    \caption{\textbf{Base model: \texttt{Qwen3-1.7B}, Reward model: Skywork (\texttt{Qwen3-8B}), Task: \texttt{IFEval}.} Figure (a) and (b) plot the per-training step policy's generation mean reward using prompts in the training dataset and validation dataset respectively---both rewards are consistently higher for GOPO-updated policies throughout training. Figure (c) contains the best benchmark score (see Section \ref{sec:eval} for details) of GOPO-updated policies and GRPO-updated policies across multiple generation temperatures---GOPO achieves higher scores at earlier checkpoints. 
    % Lastly, validation reward of GRPO at its last training step is achieved earlier for GOPO (step $300$), and the GOPO win-rate at its earlier training step against the final GRPO is $0.44$. \textcolor{blue}{(KC: llm-as-judge is not so favorable for ifeval---how to present?)}
    }
    \label{fig:ifeval-1.7b}
\end{figure*}

% \rrd{first plot needs to be cleaner -- unless the figure 3 with ultrachat gets more impactful, I will recommend you put that one in the appendix} 

% \begin{figure*}[h]
%     \centering
%     \begin{tabular}{cc}
%         \includegraphics[width=0.45\linewidth]{figure/Qwen3-1.7B-if-bsz128-ts750_best_scores_by_temp_strict.pdf} &
%         \includegraphics[width=0.45\linewidth]{figure/Qwen3-1.7B-if-bsz128-ts750_best_scores_by_temp_loose.pdf}\\
%         (a) Strict score  & (b) Loose score
%     \end{tabular}
%     \caption{\textbf{Benchmark evaluation (\texttt{IFEval}).}}
%     \label{fig:ifeval-bench}
% \end{figure*}

% \textcolor{blue}{@TODO: Try different post-training data and different baseline policy model. Scale up prompt batch size and completion size.}

\newcommand{\chat}{\texttt{UltraChat}}
\newcommand{\tldr}{\texttt{TLDR}}

% \textcolor{purple}{@WJ: Using the words and the numbers, can you reformat the experiment section that is more readable for the LLM community?}

We systematically compare the text generation quality of language models post-trained by either GOPO or GRPO, for a suite of tasks and model scales. 

% This section first elaborates on the model training setting and evaluation metrics, then discusses the experiment results. For a suite of baseline models, tasks and evaluation metrics, GOPO achieves improved performance.

% 
% An out-of-sample reward model based win-rate is introduced as our main evaluation metric; the metric is elaborated in \ref{sec:eval} and it provides a fair comparison on the two different update strategies applied on a baseline policy. 

\subsection{Training}
\paragraph{Models}

We use three model sizes: \texttt{Qwen3-1.7B}, \texttt{Qwen3-4B}, and \texttt{Qwen3-8B}. We employ instruction-tuned variants to ensure strong baseline performance and stable training. Each baseline model is updated using either GRPO or GOPO, with the number of training steps varying by model size. At each step, we sample 128 prompts ($B$) and sample 8 completions ($G$) per prompt. Let $\pi_{\mathrm{gopo}}^\star(k)$ and $\pi_{\mathrm{grpo}}^\star(k)$ denote the $k$th checkpoint, that is, the policy after $k$ training steps, obtained via GOPO and GRPO respectively. We use a learning rate of $1 \times 10^{-6}$ with 10 warmup steps. The maximum generation length is capped at 2048 tokens.

\paragraph{KL-adjusted training steps}

We adopt a model-agnostic KL budget to calibrate training across different baseline model sizes. Larger baseline models naturally require fewer training steps, as they are already more capable.

KL divergence serves not only as an explicit budget during post-training, but also as a drift budget that can be consumed fairly across policies. We therefore apply a principled early-stopping criterion for larger baseline models as follows. 

For each task of interest, the KL of $500$ training steps using \texttt{Qwen3-1.7B} is set as $\mrm{KL}^\star$, and we stop training at $k^\star = \min_{k} \big\{k: \mrm{KL}(\pi_{\mrm{gopo}}^\star(k) \| \pi_{\mrm{ref}}) \geq \mrm{KL}^\star \big\}$. In practice, larger models (\texttt{Qwen3-4B} and \texttt{Qwen3-8B}) reach the same KL divergence as \texttt{Qwen3-1.7B} in roughly half the number of training steps.

% Let policy $\pi_{\mrm{std}}$ be the final optimizer for $\mc J_{\mrm{GRPO}}$ using the standardized advantages (\ref{eq:std-advantage}). Our proposed policy $\rankpolicy$ is the final optimizer for $\mc J_{\mrm{GOPO}}$ using the rank based advantages $\hat A_{i, t} = \hat A_{i, t}^{\mrm{rank}}$.

\paragraph{Tasks and reward models}

We train and evaluate on three tasks: summarization, chat completion, and instruction following. For summarization and chat completion, we use TL;DR~\cite{trl_lib_tldr_dataset} and UltraChat~\cite{ding2023enhancing} for training, and evaluate on held-out in-distribution test sets. For instruction following, we use Tulu-3~\cite{lambert2024tulu} for training and IFEval~\cite{zhou2023instruction} for evaluation. For convenience, we refer to these datasets as \texttt{TLDR}, \texttt{UltraChat}, and \texttt{IFEval}.

We employ two open-source external reward models that rank highly on RewardBench~\cite{lambert2025rewardbench}, which evaluates models across multiple criteria, including question answering, instruction following, and fact checking. \texttt{UltraChat} uses QRM (\texttt{LLama-8B})~\cite{dorka2024quantile} or Skywork (\texttt{Qwen3-8B})~\cite{liu2025skywork}, while \texttt{TLDR} and \texttt{IFEval} use Skywork (\texttt{Qwen3-8B}).

% We first elaborate on the design and result of our experiments which compares the performance of the two policies updated by GOPO and GRPO. The experiment is extends on three different axes: size of baseline policy models, types of tasks/reward model pairs and evaluation metrics (i.e., different judges for LLM-as-judge). 

% \textcolor{blue}{TODO: Explain how we do evaluation --- majority vote, prompt input, various judges, distribution of majority wins, ... UPDATE: ... since we explain this, how will we talk about ties, what is best, need to ask raaz i think}    

% \textcolor{purple}{TODO: talk about reward trajectories, gradient norms, etc}

% \textcolor{magenta}{TODO: Add the specifics of the experiment here, not in the following sub-sections.}

\subsection{Evaluation}\label{sec:eval}

\iffalse
\begin{itemize}
    \item \texttt{n-gram} based matching metric \cite{papineni2002bleu,lin2004rouge} -- \textcolor{blue}{Nvidia guy said that these metrics are gaining interest again, especially in the \textit{translation} tasks}
    \item Regression based metric --- this requires human labeling, but as @DS mentioned, human evaluation need be done like a survey, meaning that there should several metrics (correctness, conciseness, ...) that need be instructed when API calling
    \item 
    \item LLM-as-judge: unsupervised and easy to scale --- also the prompt need be carefully designed as done in regression based metric above. Further we can apply CLT uncertainty quantification (shared by @DS in this \href{https://www.anthropic.com/research/statistical-approach-to-model-evals}{article}): Given $V$ test prompts, $q_1, ..., q_V$, consider bootstrapping subsamples $q_1^{(b)}, ..., q_n^{(b)}$ for $b = 1, ..., B$. Then define $\mbf{win}_{\mrm{llm}}^{(b)} = n^{-1}\sum_{j= 1}^{n}\mbf 1 \{ \text{model chooses $o_{j, b}^{\mrm{gopo}}$} \}$, where $o_{j, b}^{\mrm{gopo}}, o_{j, b}^{\mrm{grpo}}$ are the completions. Then CLT implies $\mbf{win}_{\mrm{llm}}^{(b)}$, $b = 1, ..., B$ would follow a Gaussian distribution --- after appropriate normalization, let's try to apply hypothesis testing. 
\end{itemize}
\fi

Evaluating open-ended text generation is inherently challenging. We therefore conduct both pointwise and pairwise evaluations for GRPO and GOPO using multiple metrics, aiming to provide high-fidelity judgments under our evaluation constraints.

\iffalse
\paragraph{Out-of-sample reward model win-rate}

The rewards $r^{\mrm{gopo}}_v = r_\phi(q_v, o^{\mrm{gopo}}_v)$ and $r^{\mrm{grpo}}_v = r_\phi(q_v, o^{\mrm{grpo}}_v)$ are calculated. For the validation prompt, we define a win of GOPO if $r^{\mrm{gopo}}_v > r^{\mrm{grpo}}_v$. The reward model based win-rate of GOPO is formally defined as 
\begin{equation}\label{eq:gopo-win}
    \begin{aligned}
        \mbf{win}_{\mrm{reward}} = \frac{1}{V}\sum_{v = 1}^{V} \mbf 1 \{ r_v^{\mrm{gopo}} > r_v^{\mrm{grpo}} \}.
    \end{aligned}
\end{equation}
The reward model based win-rate $\mbf{win}_{\mrm{reward}}$ is a fair comparison metric between the two policies $\pigopo$ and $\pigrpo$. Asides from how the advantage $\hat A_{i, t}$ was defined, everything else in the training are identically configured including the reward model $r_\phi$. Further, the triples $(q, o_{\mrm{gopo}}, o_{\mrm{grpo}})$ are independently sampled from the training data set used to attain $\pigopo, \pigrpo$. 
\fi

\paragraph{Training and validation reward}

At a fixed training step, a batch of $B$ training prompts $q$ is sampled, and both GOPO and GRPO generate $G$ completions $o_i$ per prompt from the current policy. An external reward model $r_\phi$ assigns a reward to each prompt–completion pair, $r_i = r_\phi(q, o_i)$. The training reward is computed as the average, across $B$ prompts, of the per-prompt mean reward $G^{-1} \sum_{i=1}^{G} r_\phi(q, o_i)$. The validation reward is defined identically, except that the $B$ prompts are sampled from a held-out validation set.

\paragraph{LLM-as-judge win-rate}

For pairwise evaluation, we use an LLM-as-judge to determine the winner between GRPO and GOPO. We randomly sample prompts $q_v$, for $v = 1, \ldots, V$. Each prompt $q_v$, together with its system prompt if applicable, is passed to two trained policies, $\pi_{\mathrm{gopo}}(k)$ and $\pi_{\mathrm{grpo}}(k')$, producing responses $o^{\mathrm{gopo}}_v(k)$ and $o^{\mathrm{grpo}}_v(k')$, respectively. 
% The win rate is estimated via bootstrapping over the test prompts.
% \textcolor{blue}{@KS should we explain bootstrapping more in detail?} \textcolor{blue}{@WJ Probably not, since we are not constructing confidence intervals---we might have to change the last sentence to simple averaging.}

We use \texttt{gpt5} from OpenAI as the judge. For each comparison, we provide the triple $(q_v, o^{\mathrm{gopo}}_v(k), o^{\mathrm{grpo}}_v(k'))$ along with a rubric, and the judge selects a winner between the two responses. Note that the checkpoints $k$ and $k'$ need not be identical.
To mitigate positional bias, where the judge may prefer the first presented response, we randomize the presentation order of $o^{\mathrm{gopo}}_v(k)$ and $o^{\mathrm{grpo}}_v(k')$ for each judgment~\cite{wang2024large}.

For more fine-grained evaluation, the rubric consists of five criteria: helpfulness, correctness, coherence, complexity, and verbosity. The full rubric is provided in Appendix \ref{app:judge}. The judge selects a winner for each criterion, and the final winner is determined by majority vote across the five criteria. The win-rate is defined as the proportion of prompts for which GOPO is selected as the overall winner.

\paragraph{Benchmark score}

The \texttt{IFEval} task is partially verifiable by design, with instructions that include explicit constraints such as prohibiting commas or enforcing length limits. The benchmark consists of 571 test prompts that cover 25 different instruction types. Although this setup does not necessarily assess the semantic quality of the generated responses, it enables pointwise and more easily quantifiable evaluation through a deterministic verifier.

\section{Results}

We demonstrate the superiority of GOPO over GRPO using three evaluation methods.
\begin{enumerate}
\item Throughout training, we track both training and validation rewards for policies updated via GOPO and GRPO.
\item We report LLM-as-judge~\cite{gu2024survey} win rates (over multiple generation seeds) of GOPO-updated policies over GRPO-updated policies across multiple intermediate checkpoints. See Figures \ref{fig:tldr-8b}, \ref{fig:chat-4b} and \ref{fig:tldr-1.7b} for details. Additional experiment results are in the Appendix \ref{app:eval-extra}.
\item When available, we report pointwise benchmark scores~\cite{zhou2023instruction}. See Figure \ref{fig:ifeval-1.7b} for details.
\end{enumerate}
Overall, we show that our method is robust at improving performance across model sizes, tasks, evaluation metrics, and multiple sampling temperatures. See Table \ref{tab:qwen8b_combined} and Table \ref{tab:qwen_small_combined} in Appendix \ref{app:robust} for additional robust properties of GOPO across sampling temperatures. Furthermore, we provide text examples where GOPO generates higher-quality outputs than GRPO across a range of training steps; see Appendix~\ref{app:examples} for details.

% \textcolor{blue}{@KS should probably say our methods are robust across different sampling tempratures, model sizes, tasks, etc rather than just tempratures....}

\subsection{Experiment result: \texttt{TLDR} and \texttt{UltraChat} LLM-as-judge}\label{sec:tldrchat}

For the two non-verifiable tasks \texttt{TLDR} and \texttt{UltraChat}, across multiple model sizes (\texttt{Qwen3-1.7B, 4B, 8B}), (i) we compare the training / validation reward trajectories of GOPO and GRPO updates, and (ii) track the LLM-as-judge GOPO win-rates. GOPO improves on all the evaluation metrics and does so across multiple sampling temperatures.

\paragraph{LLM-as-judge win-rates}

We track the GOPO win-rates against GRPO at identical training steps, i.e., for some fixed training step $k$, we ask \texttt{gpt-5} the winner among the two generations $o_v^{\mrm{gopo}}(k)$ and $o_v^{\mrm{grpo}}(k)$; see Section \ref{sec:eval} for details.

Under a fixed sampling temperature $\tau = 0.5$, Figures \ref{fig:tldr-8b}-(c), \ref{fig:chat-4b}-(c) and \ref{fig:tldr-1.7b}-(c) imply that the win-rates of GOPO updated policies are consistently above $0.5$ throughout training, with statistical significance by constructing confidence intervals using multiple seeds. At its best, GOPO achieves a win-rate near $0.6$. Thus, regardless of when the training is terminated, the end policy updated by GOPO is guaranteed to either tie or win against that of GRPO. Additional experiment results can be found in Appendix \ref{app:eval-extra}.

% This observation is consistent across all sampling temperatures, hence exhibiting robust performance of GOPO; see Appendix \ref{app:tldr-extra} for more results. 

\paragraph{Training / Validation reward and efficiency}

The training and validation reward trajectories for GOPO updates are consistently above those of GRPO updates as seen in Figures \ref{fig:tldr-8b}-(a), (b), \ref{fig:chat-4b}-(a), (b) and \ref{fig:tldr-1.7b}-(a), (b); see Section \ref{sec:eval} for details on how training / validation rewards are defined.

Validation rewards are often used as a proxy for comparing policies, so we compare the policy trained by GOPO where its validation reward matches the best validation reward of GRPO updated policies---in Figures \ref{fig:tldr-8b}-(b), \ref{fig:chat-4b}-(b) and \ref{fig:tldr-1.7b}-(b), the blue horizontal dotted line refers to GRPO's best validation reward and the red vertical dashed line refers to the early-stopping training step of GOPO updates. We observe that GOPO achieves GRPO's highest validation reward at approximately half of the training steps, while achieving a comparable LLM-as-judge win-rate (slightly above $0.5$). In other words, GOPO is able to reach a policy distribution comparable (in validation reward and llm-as-judge win-rate) to that of GRPO $2\times$ faster.

% \textcolor{blue}{Message: Given the same (model agnostic) KL or Trust Region budget, (1) GOPO achieves it faster with fewer training steps with comparable win-rates and (2) in some cases, performance is better. Giving the same KL for different update schemes, you are giving equal allowance to edit the baseline model's distribution---how fast can you edit to a distribution of comparable capacity?}

% We compare the completions $o_v^{\mrm{gopo}}(k)$ and $o_v^{\mrm{grpo}}(k')$ when $\pigopo(k)$ and $\pigrpo(k')$ achieve the same KL divergence from the reference policy. Normally, the gradient norm for GOPO is larger than that of GRPO update, hence implying a faster divergence of GOPO updated policies from the reference policy. The typical divergence of policies are measure by KL, which serves also as a regularization measure during policy updates. In general, GOPO updated policies achieve the same KL budget at half the training steps compared to GRPO updated policies. 

% \begin{figure}[t]
%     \centering
%     \includegraphics[width=1.0\linewidth]{figure/tldr_scatter_survey.pdf}
%     \caption{Scatter plot of (best win-rate, checkpoint) and (worst win-rate, checkpoint) for all temperatures.}
%     \label{fig:tldr-best-worst-scatter}
% \end{figure}

% We compare GOPO updated and GRPO updated policies of same KL divergence, specifically $\pi^{\mrm{gopo}}(250)$ and $\pi^{\mrm{grpo}}(500)$, where $500$ is the max number of training steps we train policies. In Figure \ref{fig:tldr-kl}, we observe that the win-rates of $\pi^{\mrm{gopo}}(250)$ are slightly above $0.5$, implying that the 

\paragraph{Robustness}
 
The LLM-as-judge win-rates of GOPO in Figures \ref{fig:tldr-8b}, \ref{fig:chat-4b} and \ref{fig:tldr-1.7b} are established under multi-seed generation, but at a fixed sampling temperature of $0.5$. Across multiple sampling temperatures, we show GOPO win-rate is consistently higher for the majority of training steps. Table \ref{tab:qwen8b_combined} presents the GOPO win-rates (under sampling temperatures $0.1, 0.5, 0.9$) on tasks \texttt{TLDR, Ultrachat} for a base model \texttt{Qwen3-8B}. The win-rates are consistently above $0.5$ for most training steps. See Table \ref{tab:qwen_small_combined} in Appendix \ref{app:robust} for additional results on different model sizes. 

\begin{table}[t]
\centering
\setlength{\tabcolsep}{3pt} % Reduce column spacing (default is 6pt)
\caption{\textbf{Base model: \texttt{Qwen3-8B}.} LLM-as-judge GOPO win-rates against GRPO at identical training step. Left block: Reward model Skywork (\texttt{Qwen3-8B}) on \texttt{TLDR}. Right block: Reward model Skywork (\texttt{Qwen3-8B}) on \texttt{UltraChat}. Rows show training progress; columns show sampling temperature $\tau$ of the policy trained up to the specific steps.}
\small % Optional: use smaller font
\begin{tabular}{rccc|ccc}
\toprule
& \multicolumn{3}{c}{\textbf{Task: \texttt{TLDR}}} 
& \multicolumn{3}{c}{\textbf{Task: \texttt{UltraChat}}} \\
\cmidrule(lr){2-4}\cmidrule(lr){5-7}
\textbf{Steps} 
& \textbf{$\tau$\ =\ 0.1} & \textbf{$\tau$\ =\ 0.5} & \textbf{$\tau$\ =\ 0.9}
& \textbf{$\tau$\ =\ 0.1} & \textbf{$\tau$\ =\ 0.5} & \textbf{$\tau$\ =\ 0.9} \\
\midrule
25  & 0.58 & 0.52 & 0.55 & 0.53 & 0.54 & 0.49 \\
50  & 0.56 & 0.58 & 0.57 & 0.50 & 0.46 & 0.48 \\
75  & 0.60 & 0.57 & 0.60 & 0.59 & 0.56 & 0.55 \\
100 & 0.61 & 0.58 & 0.58 & 0.52 & 0.53 & 0.50 \\
125 & 0.58 & 0.59 & 0.52 & 0.54 & 0.54 & 0.56 \\
150 & 0.57 & 0.57 & 0.60 & 0.52 & 0.47 & 0.47 \\
175 & 0.58 & 0.59 & 0.55 & 0.58 & 0.55 & 0.59 \\
200 & 0.58 & 0.54 & 0.56 & 0.59 & 0.56 & 0.56 \\
\bottomrule
\end{tabular}
\label{tab:qwen8b_combined}
\end{table}

% \begin{table}[htbp]
% \centering
% \begin{tabular}{rccc}
% \toprule
% \textbf{Steps} & \textbf{$\tau$\ =\ 0.1} & \textbf{$\tau$\ =\ 0.5} & \textbf{$\tau$\ =\ 0.9} \\
% \midrule
% 25 & 57.9 & 52.0 & 54.6 \\
% 50 & 56.0 & 58.4 & 56.8 \\
% 75 & 59.7 & 57.4 & 59.7 \\
% 100 & 61.0 & 57.8 & 58.2 \\
% 125 & 58.1 & 58.6 & 52.1 \\
% 150 & 57.0 & 57.0 & 59.5 \\
% 175 & 57.6 & 58.9 & 55.0 \\
% 200 & 57.9 & 54.2 & 56.1 \\
% \bottomrule
% \end{tabular}
% \caption{\textbf{Base model: \texttt{Qwen3-8B}, Reward model: Skywork (\texttt{Qwen3-8B}), Task: \texttt{TLDR}.} LLM-as-judge GOPO win-rates (\%), rows show training progress, columns show sampling temperature. GOPO win-rates are above $0.5$ (near $0.6$ at its best) across all training progress and sampling temperature $\tau$.}
% \label{tab:qwen8b_tldr}
% \end{table}

% \begin{figure}[t]
%     \centering
%     \includegraphics[width=1.0\linewidth]{figure/tldr_qwen4B_seed42.pdf}
%     \caption{}
%     \label{fig:tldr-kl}
% \end{figure}

\subsection{Experiment result: \texttt{IFEval} benchmark}\label{sec:ifeval}

For the task \texttt{IFEval}, (i) we compare the training / validation reward trajectories of GOPO and GRPO updates, and (ii) track the benchmark score for GOPO and GRPO updated policies \cite{zhou2023instruction}. GOPO improves on all evaluation metrics, across multiple sampling temperatures.

\paragraph{Benchmark scoring and robustness}

There are two types of scores that are provided by the benchmark~\cite{zhou2023instruction}, a strict score and a loose score. Given $N$ test prompts $q_j$ of \texttt{IFeval}, a completion $o_j$ from some policy is assigned score $1$ if all the instructions in the prompt $q_j$ are satisfied for $o_j$, then these binary values are averaged across all $N$ prompts to define the benchmark score. Two types of scores are further explained in the Appendix.

A policy positioned in the upper-right quadrant in Figure \ref{fig:ifeval-1.7b}-(c) is considered most desirable as it implies a faster convergence of the base model to a capable model for instruction-following task. The best version (out of multiple training steps) of GOPO updated policies stays within the upper-right quadrant uniformly across all sampling temperatures. 

\paragraph{Training / Validation reward}

The training and validation reward trajectories for GOPO updates are consistently above that of GRPO updates as seen in Figures \ref{fig:ifeval-1.7b}-(a), (b); see Section \ref{sec:eval} for details on how training / validation rewards are defined.

\section{Conclusion}
Reward models in RLHF are fitted to pairwise preferences, so their most reliable signal is ordinal, which response is better, while many policy optimization methods still assume the reward scale is meaningful. We introduced Group Ordinal Policy Optimization, a replacement for GRPO that removes this mismatch by discarding reward magnitudes and using only within-prompt reward ranks. Essentially, the policy update depends on the same comparisons that the reward model is trained to make.

Empirically, GOPO delivers consistently stronger and more sample-efficient post-training than GRPO across tasks, model sizes, reward models, and sampling temperatures. It improves training/validation reward trajectories, yields favorable LLM-as-judge win rates, and achieves quicker gains on the partially verifiable IFEval benchmark (Figures \ref{fig:tldr-8b}, \ref{fig:chat-4b}, \ref{fig:tldr-1.7b}, \ref{fig:ifeval-1.7b}; Table \ref{tab:qwen8b_combined}; Appendix \ref{app:eval-extra-robust}). Our analysis helps to connect these outcomes to the algorithmic choice. That is, rank-based advantages typically increase the norm of policy gradient update, especially for small group sizes $G$. This is consistent with faster KL-budget consumption and faster convergence during training. Qualitative examples further suggest that the same mechanism improves robustness in ways that matter for alignment, such as better calibration and fewer unsupported specifics. Overall, GOPO provides a simple and effective mechanism for RLHF in the non-verifiable reward setting.

% \section{Conclusion}
% Reward models in RLHF are trained on pairwise preferences, so their most reliable signal is ordinal—which response is better—yet many policy-optimization methods still treat reward scores as calibrated, cardinal values. We introduced GOPO, a replacement for GRPO that discards reward magnitudes and uses only within-prompt reward ranks, aligning the policy update with the comparisons the reward model is trained to make.

% Empirically, GOPO delivers stronger and more sample-efficient post-training than GRPO across tasks, model sizes, reward models, and sampling temperatures: it improves training/validation reward trajectories, yields favorable LLM-as-judge win rates, and achieves earlier gains on the partially verifiable IFEval benchmark (Figures \ref{fig:tldr-8b}, \ref{fig:chat-4b}, \ref{fig:tldr-1.7b}, \ref{fig:ifeval-1.7b}; Table \ref{tab:qwen8b_combined}; Appendix \ref{app:eval-extra-robust}). Our analysis connects these gains to the rank transform: it typically increases update norms at small group sizes $G$, consistent with faster KL-budget usage and quicker progress (Theorem \ref{thm:B1-gopo-inflation}), while remaining well-behaved as the number of samples per prompt grows (Appendix~\ref{app:thm-large-G}--\ref{app:kl}). Qualitative examples further suggest improved calibration and fewer unsupported specifics. Overall, GOPO provides a simple and effective principle for RLHF in non-verifiable settings: \textbf{treat reward-model outputs as ordinal}.

\newpage


\bibliography{ref}
\bibliographystyle{icml2026}

%%%%%%%%%%%%%%%%%%%%%%%%%%%%%%%%%%%%%%%%%%%%%%%%%%%%%%%%%%%%%%%%%%%%%%%%%%%%%%%
%%%%%%%%%%%%%%%%%%%%%%%%%%%%%%%%%%%%%%%%%%%%%%%%%%%%%%%%%%%%%%%%%%%%%%%%%%%%%%%
% APPENDIX
%%%%%%%%%%%%%%%%%%%%%%%%%%%%%%%%%%%%%%%%%%%%%%%%%%%%%%%%%%%%%%%%%%%%%%%%%%%%%%%
%%%%%%%%%%%%%%%%%%%%%%%%%%%%%%%%%%%%%%%%%%%%%%%%%%%%%%%%%%%%%%%%%%%%%%%%%%%%%%%
\newpage
\appendix
\onecolumn
% \section{Appendix}

\section{LLM-as-judge prompting}\label{app:judge}

We elaborate on the LLM-as-judge evaluation described in Section \ref{sec:eval}. We ask the judge to choose winners across the following $5$ criteria, then the final winner among the completions are chosen based on majority vote.

\begin{enumerate}
    \item \textit{Helpfulness}: How well does the response satisfy what the question asks for? Does it address the core needs of the prompt?
    \item \textit{Correctness}: Does the response contain factually accurate and relevant information? Are there any hallucinations, errors, or false information?
    \item \textit{Coherence}: Is the response clear, logical, and self-consistent? Does it flow well and make sense?
    \item \textit{Complexity}: What is the level of intellectual depth and sophistication? Consider vocabulary, sentence structure, and whether the response demonstrates basic or expert-level understanding.
    \item \textit{Verbosity}: Is the response appropriately concise or detailed relative to what the question asks for? Is it too brief, too verbose, or just right?
\end{enumerate}

\section{Additional evaluations}\label{app:eval-extra-robust}

\subsection{Reward trajectores and LLM-as-judge}\label{app:eval-extra}

\begin{figure*}[!ht]
    \centering
    \begin{tabular}{ccc}
        \includegraphics[width=0.32\linewidth]{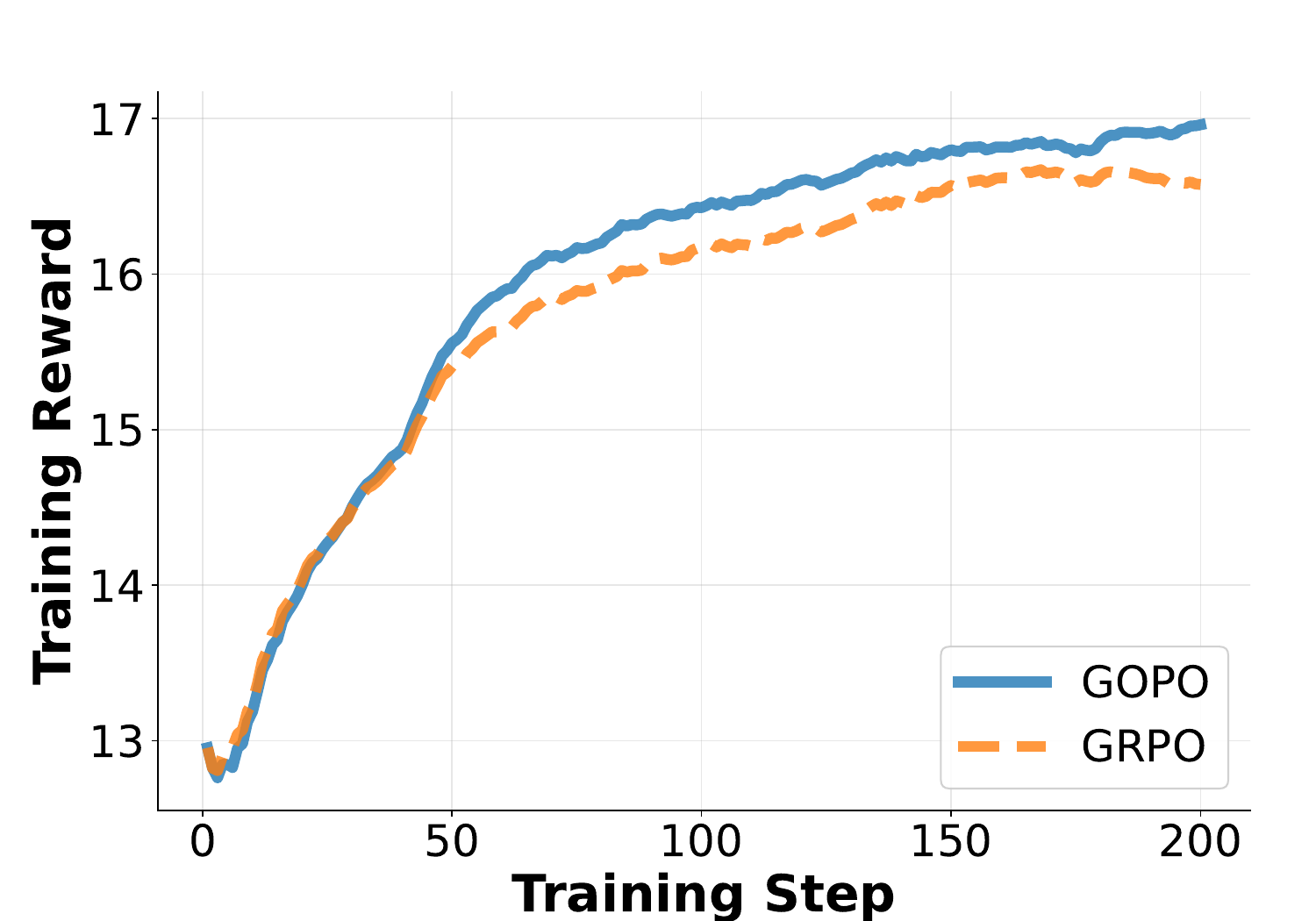} &
        \includegraphics[width=0.32\linewidth]{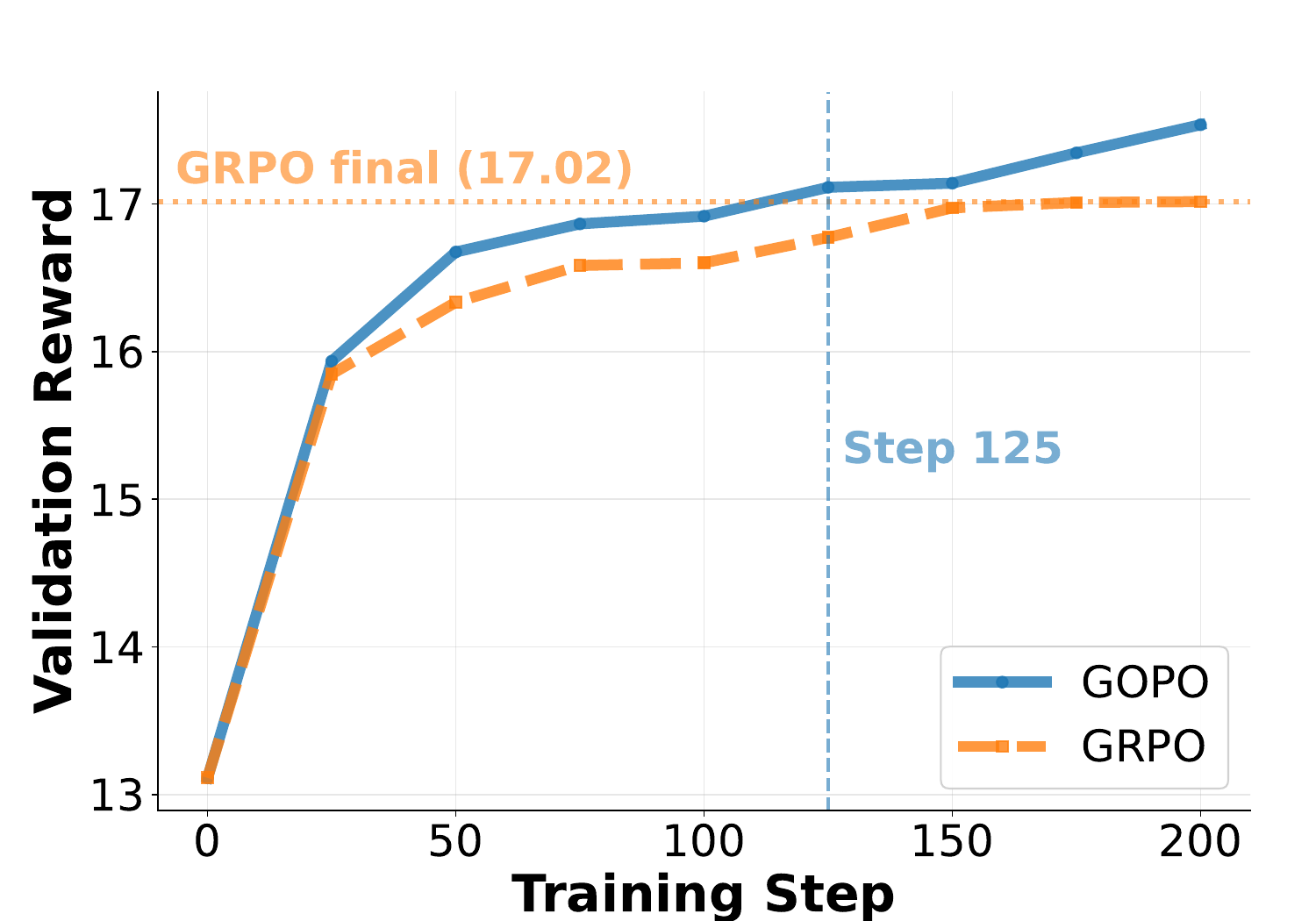} &
        \includegraphics[width=0.32\linewidth]{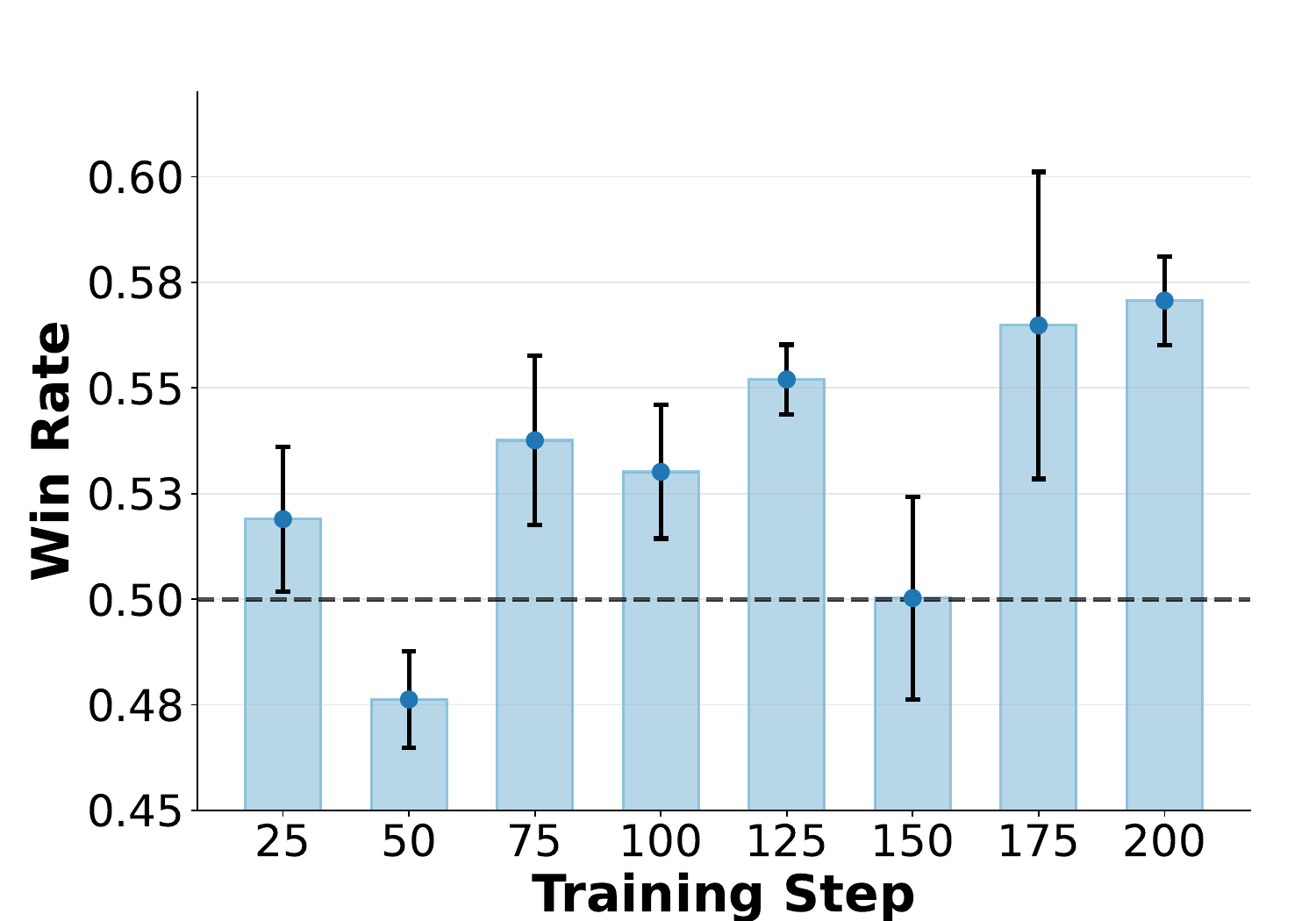} \\
        (a) Training Reward &
        (b) Validation Reward &
        (c) LLM-as-judge Evaluation
    \end{tabular}
    \caption{\textbf{Base model: \texttt{Qwen3-8B}, Reward model: Skywork (\texttt{Qwen3-8B}), Task: \texttt{UltraChat}.} Training reward (a) and validation reward (b) is the per-training step policy's generation mean reward using prompts in the training dataset and validation dataset respectively---both rewards are consistently higher for GOPO throughout training. Figure (c) refers to the LLM-as-judge (\texttt{gpt-5}) win-rate of GOPO updated policies against GRPO updated policies at the identical training steps---for \emph{multi-seed generations}, GOPO improves the win-rates throughout most training steps. The policy generation temperature is fixed at $0.5$; see Table \ref{tab:qwen8b_combined} in Section \ref{sec:tldrchat} for results on varying temperatures. Validation reward of GRPO at its last training step is achieved earlier for GOPO (step $125$), and the GOPO win-rate at its earlier training step against the final GRPO is $0.53$.}
    \label{fig:chat-8b}
\end{figure*}

\begin{figure*}[!ht]
    \centering
    \begin{tabular}{ccc}
        \includegraphics[width=0.32\linewidth]{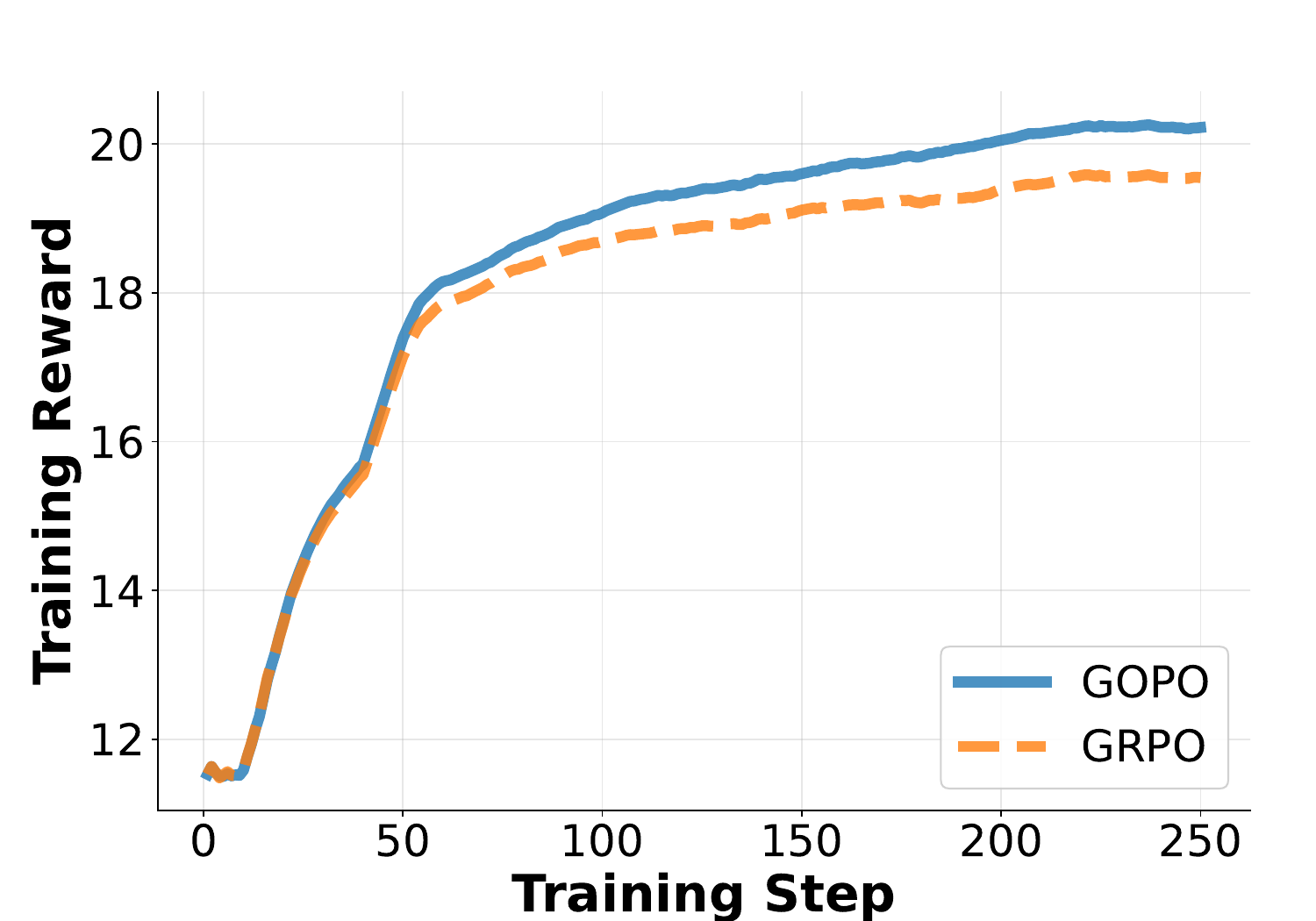} &
        \includegraphics[width=0.32\linewidth]{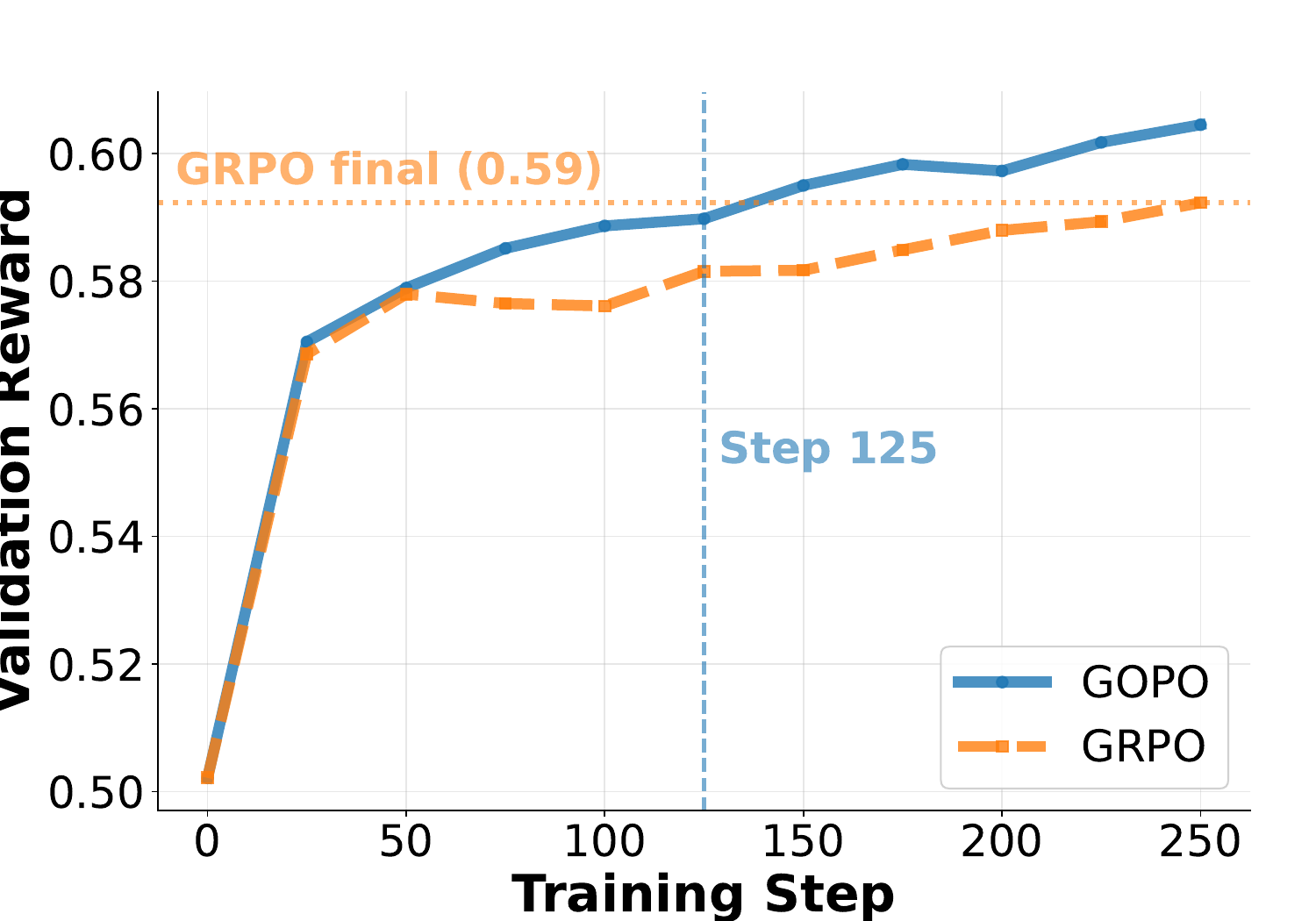} &
        \includegraphics[width=0.32\linewidth]{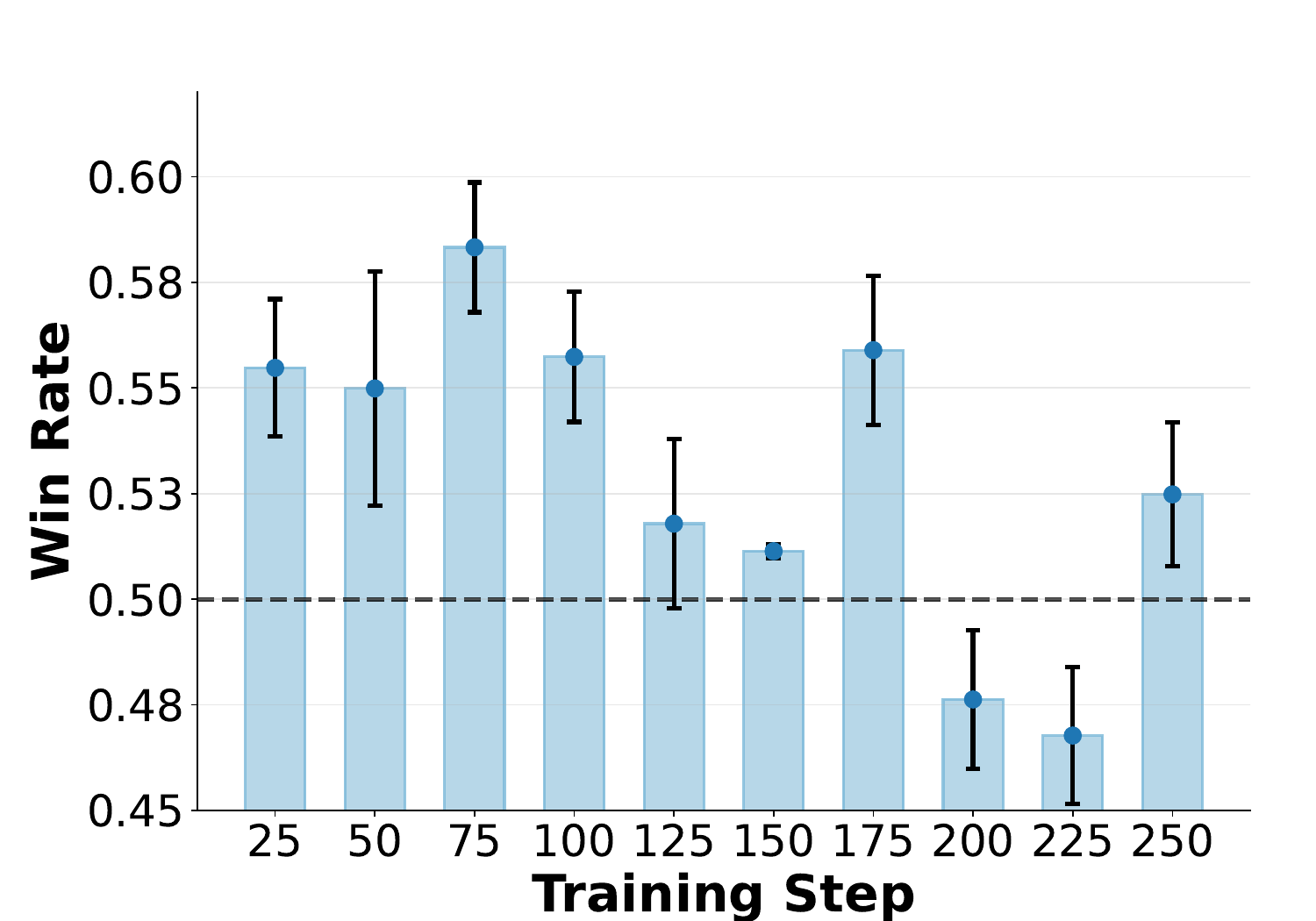} \\
        (a) Training Reward &
        (b) Validation Reward &
        (c) LLM-as-judge Evaluation
    \end{tabular}
    \caption{\textbf{Base model: \texttt{Qwen3-4B}, Reward model: Skywork (\texttt{Qwen3-8B}), Task: \texttt{TLDR}} 
    Training reward (a) and validation reward (b) is the per-training step policy's generation mean reward using prompts in the training dataset and validation dataset respectively---both rewards are consistently higher for GOPO throughout training. Figure (c) refers to the LLM-as-judge (\texttt{gpt-5}) win-rate of GOPO updated policies against GRPO updated policies at the identical training steps---for \emph{multi-seed generations}, GOPO improves the win-rates throughout most training steps. The policy generation temperature is fixed at $0.5$; see Table \ref{tab:qwen_small_combined} in Appendix \ref{app:robust} for results on varying temperatures. Validation reward of GRPO at its last training step is achieved earlier for GOPO (step $125$), and the GOPO win-rate at its earlier training step against the final GRPO is $0.50$.} 
    \label{fig:tldr-4b}
\end{figure*}

% \begin{figure*}[!ht]
%     \centering
%     \begin{tabular}{ccc}
%         \includegraphics[width=0.32\linewidth]{figure/tr_ultrachat_qwen1.7b_qrm_seed42.pdf} &
%         \includegraphics[width=0.32\linewidth]{figure/val_ultrachat_qwen1.7b_qrm_seed42.pdf} &
%         \includegraphics[width=0.32\linewidth]{figure/qwen1.7b_chat_qrm_N25_B25_temp0.5_survey_agg_ci.pdf} \\
%         (a) Training Reward &
%         (b) Validation Reward &
%         (c) LLM-as-judge Evaluation
%     \end{tabular}
%     \caption{\textbf{Base model: \texttt{Qwen3-1.7B}, Reward model: QRM~(\texttt{Llama-8B}), Task: \texttt{UltraChat}.} Training reward (a) and validation reward (b) is the per-training step policy's generation mean reward using prompts in the training dataset and validation dataset respectively---both rewards are consistently higher for GOPO throughout training. Figure (c) refers to the LLM-as-judge (\texttt{gpt-5}) win-rate of GOPO updated policies against GRPO updated policies at the identical training steps---GOPO consistently improves the win-rates throughout most training steps. The policy generation temperature is fixed at $0.7$. Validation reward of GRPO at its last training step is achieved earlier for GOPO (step $400$), and the GOPO win-rate at its earlier training step against the final GRPO is $0.51$.}
%     \label{fig:chat-1.7b}
% \end{figure*}

\newpage

\subsection{LLM-as-judge across sampling temperatures}\label{app:robust}

We refer to Table \ref{tab:qwen_small_combined} for additional LLM-as-judge GOPO win-rates for smaller sized base models.

\begin{table*}[h]
\centering
\caption{\textbf{LLM-as-judge GOPO win-rates against GRPO at identical training progress (\%).} Left block: Base model \texttt{Qwen3-1.7B}. Right block: Base model \texttt{Qwen3-4B}. Reward model Skywork (\texttt{Qwen3-8B}) used for all experiments except \texttt{Qwen3-1.7B} on \texttt{UltraChat}, which uses QRM (\texttt{Llama-8B}). Columns show sampling temperature $\tau$.}
\small
\setlength{\tabcolsep}{3pt}
\begin{tabular}{rccc|ccc|ccc|ccc}
\toprule
& \multicolumn{6}{c|}{\textbf{Base model: \texttt{Qwen3-1.7B}}} 
& \multicolumn{6}{c}{\textbf{Base model: \texttt{Qwen3-4B}}} \\
\cmidrule(lr){2-7}\cmidrule(lr){8-13}
& \multicolumn{3}{c|}{\textbf{\texttt{TLDR}}} 
& \multicolumn{3}{c|}{\textbf{\texttt{UltraChat}}} 
& \multicolumn{3}{c|}{\textbf{\texttt{TLDR}}} 
& \multicolumn{3}{c}{\textbf{\texttt{UltraChat}}} \\
\cmidrule(lr){2-4}\cmidrule(lr){5-7}\cmidrule(lr){8-10}\cmidrule(lr){11-13}
\textbf{Progress} 
& $\tau$ = 0.1 & $\tau$ = 0.5 & $\tau$ = 0.9
& $\tau$ = 0.1 & $\tau$ = 0.5 & $\tau$ = 0.9
& $\tau$ = 0.1 & $\tau$ = 0.5 & $\tau$ = 0.9
& $\tau$ = 0.1 & $\tau$ = 0.5 & $\tau$ = 0.9 \\
\midrule
10\%  & 0.53 & 0.51 & 0.54 & 0.51 & 0.48 & 0.55 & 0.57 & 0.54 & 0.56 & 0.52 & 0.47 & 0.50 \\
20\%  & 0.53 & 0.54 & 0.54 & 0.57 & 0.55 & 0.47 & 0.58 & 0.52 & 0.55 & 0.49 & 0.51 & 0.55 \\
30\%  & 0.57 & 0.54 & 0.58 & 0.50 & 0.51 & 0.50 & 0.59 & 0.60 & 0.58 & 0.52 & 0.53 & 0.48 \\
40\%  & 0.59 & 0.62 & 0.58 & 0.54 & 0.52 & 0.50 & 0.54 & 0.55 & 0.55 & 0.52 & 0.53 & 0.53 \\
50\%  & 0.56 & 0.54 & 0.53 & 0.51 & 0.51 & 0.52 & 0.50 & 0.52 & 0.53 & 0.52 & 0.52 & 0.51 \\
60\%  & 0.56 & 0.57 & 0.56 & 0.51 & 0.52 & 0.49 & 0.50 & 0.51 & 0.52 & 0.59 & 0.60 & 0.53 \\
70\%  & 0.52 & 0.52 & 0.59 & 0.51 & 0.48 & 0.49 & 0.56 & 0.57 & 0.53 & 0.54 & 0.56 & 0.49 \\
80\%  & 0.57 & 0.51 & 0.55 & 0.49 & 0.51 & 0.54 & 0.50 & 0.49 & 0.48 & 0.55 & 0.53 & 0.51 \\
90\%  & 0.56 & 0.57 & 0.55 & 0.48 & 0.56 & 0.48 & 0.52 & 0.46 & 0.49 & 0.58 & 0.55 & 0.52 \\
100\% & 0.53 & 0.58 & 0.58 & 0.54 & 0.50 & 0.52 & 0.51 & 0.54 & 0.49 & 0.54 & 0.52 & 0.55 \\
\bottomrule
\end{tabular}
\label{tab:qwen_small_combined}
\end{table*}

\section{Proof of Theorem \ref{thm:B1-gopo-inflation}}

Recall the conditioning event $\mc F=\sigma(q,A_1,\dots,A_G)$ and the (conditionally) centered feature vector $\widetilde X_i:=X_i-\E[X_i\mid \mc F]$. As $A_i$ is $\mc F$-measurable, we observe
\begin{equation}\label{eq:xi}
    \E[g\mid \mc F]
=\E\bigg[\frac{1}{G}\sum_{i=1}^G A_i X_i \bigg| \mc F \bigg]
=\frac{1}{G}\sum_{i=1}^G A_i\,\E[X_i\mid \mc F] \implies \xi:= g- \E[g\mid\mc F]
 = \frac{1}{G}\sum_{i=1}^G A_i\widetilde X_i.    
\end{equation}
Expand the squared norm of $\xi$ expressed as in (\ref{eq:xi}), and then apply the conditional orthogonality assumption, i.e., $\E[\langle \widetilde{X}_i, \widetilde{X}_j|\mc F] = 0$ for $i \neq j$, and the conditional variance assumption, i.e., $\E[\|\widetilde X_i\|^2\mid \mc F]=\sigma_X^2$ for all $i$, to observe 
\begin{equation}\label{eq:xi-2}
    \E[\|\xi\|^2\mid \mc F] = \frac{1}{G^2}\sum_{i=1}^G\sum_{j=1}^G A_iA_j\,
    \E\big[\langle \widetilde X_i,\widetilde X_j\rangle | \mc F\big] = \frac{1}{G^2}\sum_{i=1}^G A_i^2\,\sigma_X^2=\frac{\sigma_X^2}{G}\cdot \frac{1}{G}\sum_{i=1}^G A_i^2.    
\end{equation}
% By the conditional orthogonality assumption, the $i\neq j$ terms vanish, and by the conditional variance assumption,
% $\E[\|\widetilde X_i\|^2\mid \mc F]=\sigma_X^2(q)$ for all $i$. Therefore
% \[
% \E[\|\xi\|^2\mid \mc F]
% =
% \frac{1}{G^2}\sum_{i=1}^G A_i^2\,\sigma_X^2(q)
% =
% \frac{\sigma_X^2(q)}{G}\cdot \frac{1}{G}\sum_{i=1}^G A_i^2.
% \]
As a last step, we take the expectation over the remaining randomness in equation (\ref{eq:xi-2}), yielding
\[
\E[\|\xi\|^2]
=
\frac{1}{G}\,\E\bigg[\sigma_X^2\cdot \frac{1}{G}\sum_{i=1}^G A_i^2\bigg].
\]

\subsection{Applying Theorem \ref{thm:B1-gopo-inflation} for GOPO and GRPO}
We specialize Theorem \ref{thm:B1-gopo-inflation} to GOPO and GRPO using the following Lemma, so as to observe the gradient variance inflation brought by GOPO (see Section \ref{sec:why-rank} for details).

\begin{lemma}[Empirical Second Moment of Advantages]\label{lemma:helper-advantage-bounds}
    Recall the advantage definitions for GOPO and GRPO in (\ref{eq:rank-advantage}) and (\ref{eq:std-advantage}) respectively, where we drop index $t$ as both advantages are broadcasted across tokens. Then we have
\[
\frac{1}{G}\sum_{i = 1}^{G} (\hat A_i^{\mrm{std}})^2=1 \quad \text{and} \quad \frac{1}{G}\sum_{i = 1}^{G} (\hat A_i^{\mrm{rank}})^2=\frac{4(G+1)}{3(G-1)}.
\]

\end{lemma}
\begin{proof}[Proof of Lemma \ref{lemma:helper-advantage-bounds}] Recall the definition of ranked advantages (\ref{eq:rank-advantage}) for GOPO. Setting $\Delta = 4/(G - 1),$ we observe
\begin{equation}\label{eq:rank-ad-exp}
\begin{aligned}
\frac{1}{G}\sum_{i=1}^G (\hat A_i^{\mrm{rank}})^2
&= \frac{1}{G}\sum_{k=0}^{G-1} (2-k\Delta)^2\\
&= \frac{1}{G}\sum_{k=0}^{G-1} \big(4 - 4k\Delta + k^2\Delta^2\big)\\
&= 4 - 4\Delta\bigg(\frac{1}{G}\sum_{k=0}^{G-1}k\bigg)
+ \Delta^2\bigg(\frac{1}{G}\sum_{k=0}^{G-1}k^2\bigg).
\end{aligned}
\end{equation}
Using $\sum_{k=0}^{G-1}k=G(G-1)/2$ and $\sum_{k=0}^{G-1}k^2=G(G-1)(2G-1)/6$ on (\ref{eq:rank-ad-exp}), we obtain
\begin{equation}\label{eq:rank-ad-exp-2}
    \frac{1}{G}\sum_{i=1}^G (\hat A_i^{\mrm{rank}})^2
= 4 - 2\Delta(G-1) + \Delta^2\frac{(G-1)(2G-1)}{6}.
\end{equation}
Substituting $\Delta=\frac{4}{G-1}$ in (\ref{eq:rank-ad-exp-2}) gives our desired result
\begin{equation*}
    \frac{1}{G}\sum_{i=1}^G (\hat A_i^{\mrm{rank}})^2 =4 - 8 + \frac{16(2G-1)}{6(G-1)}
=
\frac{4(G+1)}{3(G-1)}. 
\end{equation*}

Next, recall the z-score advantages $\hat A_i^{\mrm{std}}=(r_i-\bar r)/s$ with $\bar r = G^{-1}\sum_{j = 1}^{G}r_j$ and $s^2=G^{-1}\sum_{j=1}^G(r_j-\bar r)^2$ for GRPO. Then we observe
\[
\frac{1}{G}\sum_{i=1}^G (\hat A_i^{\mrm{std}})^2
=
\frac{1}{G}\sum_{i=1}^G \frac{(r_i-\bar r)^2}{s^2}
=
\frac{1}{s^2}\cdot \frac{1}{G}\sum_{i=1}^G (r_i-\bar r)^2
=
1.
\]
\end{proof}

\section{Gradient comparison for large sample size}\label{app:thm-large-G}

Here we provide a Theorem that provides insight to the asymptotic (as per-prompt sample size $G$ increase) behavior of gradient norms updated by GOPO and GRPO. Unlike the setting in Theorem \ref{thm:B1-gopo-inflation}, we consider a multi-batch ($B>1$) scenario, but still maintain the smoothed version of the objective function; hence we set the objective function as 
\begin{equation*}
    \begin{aligned}
        \mc J(\theta) &= \frac{1}{B}\sum_{b = 1}^{B} \frac{1}{G} \sum_{i = 1}^{G} \frac{1}{T}\sum_{t = 1}^{T} \pi_t(\theta) \hat A^{(b)}_{i, t} - \beta  \mrm{KL}(\pi_\theta \| \pi_{\mrm{ref}})\\
        &=: \mc J_{1}(\theta) - \beta  \mrm{KL}(\pi_\theta \| \pi_{\mrm{ref}})
    \end{aligned}
\end{equation*}
where $\mc J_{1}(\theta)$ refers to the triple sum in the above display. 

Note the following
\begin{equation*}
    \begin{aligned}
        \nabla_\theta \mc J (\theta) 
        &= \frac{1}{B} \sum_{b = 1}^{B} \frac{1}{G} \sum_{i = 1}^{G} \frac{1}{T} \sum_{t = 1}^{T} \pi_t(\theta) \nabla_\theta \log \pi_t(\theta) \cdot \hat A_{i, t}^{(b)}- \beta \nabla_\theta \mrm{KL}(\pi_\theta \| \pi_{\mrm{ref}})
    \end{aligned}
\end{equation*}
We consider the gradient norm of the first sum of the above display. Define the random vector $X_i^{(b)} :=T^{-1}\sum_{t = 1}^{T} \nabla \log \pi_t(\theta) \pi_t(\theta)$; note that the index $i$ and $b$ are implicit in $\pi_t(\theta)$, see (\ref{eq:obj-func}).  

\begin{theorem}[Bounded Gradient Norms]\label{thm:grad-upper-bound}
Suppose the prompts $q$ are independently sampled from distribution $P_Q$. Further assume $\| X_i^{(b)} \| \leq C$ for all $i$ and $b$\footnote{Note that $C$ would scale with $T$.}.
\begin{enumerate}
    \item If $\hat A_{i, t}^{(b)}$ is the rank of rewards (\ref{eq:rank-advantage}), then $\| \nabla \mc J_1(\theta)\|\leq 2C$ almost surely.
    \item If $\hat A_{i, t}^{(b)}$ is the standardized rewards (\ref{eq:std-advantage}) and have sub-Gaussian behavior, then $\| \nabla_\theta \mc J_1(\theta) \| \leq O_p\big(\sqrt{\log G}\big)$ as $B, G \to \infty$.
\end{enumerate}

\end{theorem}

\subsection{Proof of Theorem \ref{thm:grad-upper-bound}}

Further note that the advantage is broadcasted across tokens, meaning that $\hat A_{i, t}$ is identical across $t$, so we set $\hat A_{i}^{(b)} = \hat A_{i, t}^{(b)}$ for all $t \in [T]$. Then observe
\begin{equation}\label{eq:grad-up-pre}
    \begin{aligned}
        \frac{1}{B}\sum_{b = 1}^{B}\frac{1}{G}\sum_{i = 1}^{G}\hat A_{i}^{(b)}X_i^{(b)} \leq \frac{1}{B}\sum_{b = 1}^{B} \max_{i \in [G]}\big|\hat A_{i}^{(b)}\big| \cdot \bigg(\frac{1}{G}\sum_{i = 1}^{G} X_i^{(b)} \bigg).
    \end{aligned}
\end{equation}

Using (\ref{eq:grad-up-pre}), the gradient norm of $\nabla \mc J_1(\theta)$ can be bounded as 
\begin{equation}\label{eq:jone-norm-bound}
    \begin{aligned}
        \| \nabla_\theta \mc J_1(\theta) \| \leq \frac{1}{B}\sum_{b = 1}^B \max_{i \in[G]}|\hat A_{i}^{(b)}| \cdot \bigg( \frac{1}{G}\sum_{i = 1}^G \|X_i^{(b)}\| \bigg).
    \end{aligned}
\end{equation}
Use the assumption that $X_i^{(b)}$ have bounded norm, i.e., $\|X_i^{(b)}\|\leq C$ for some absolute so that (\ref{eq:jone-norm-bound}) is further bounded by $\|\nabla_\theta \mc J_1 (\theta) \| \leq CB^{-1}\sum_{b = 1}^{B}\max_{i \in [G]}\big| \hat A_i^{(b)} \big|$. For what follows, we bound the term $\max_{i \in [G]}\big| \hat A_{i}^{(b)} \big|$. 

We now invoke two different types of $\hat A_{i}^{(b)}$: first is the ranking used for our proposed algorithm (\ref{eq:rank-advantage}) and the second is used for the normal GRPO algorithm (\ref{eq:std-advantage}).

First, when the advantage is set as the rank of the rewards, we observe that $\hat A_{i}^{(b)} \in [-2, 2]$ for all $i$ and $b$, so (\ref{eq:jone-norm-bound}) can be further bounded as $\| \nabla_\theta \mc J_1(\theta) \| \leq 2C$.

Next, we consider the z-score advantages. For any $b \in [B]$ note the following decomposition
\begin{equation*}
    \begin{aligned}
        \max_{i \in [G]}\big| \hat A_{i}^{(b)} \big| = \max_{i \in [G]}\big| \hat A_{i}^{(b)} \big| - \E\Big[ \max_{i \in [G]}\big| \hat A_i^{(b)} \big| \Big] + \E\Big[ \max_{i \in [G]}\big| \hat A_i^{(b)} \big| \Big].
    \end{aligned}
\end{equation*}

When the advantage is set as the $z$-score of the rewards, the sub-Gaussian reward implies the following \cite{ledoux2001concentration}: for some absolute constants $c, c' > 0$, 
\begin{equation}\label{eq:max-sG}
    \begin{aligned}
        \Big\|\max_{i \in [G]}\big| \hat A_i^{(b)} \big|\Big\|_{\psi_2} \leq c \sqrt{\log G}\max_{i \in [G]}\big\| \hat A_{i}^{(b)} \big\|_{\psi_2} \leq c' \sqrt{\log G}
    \end{aligned}
\end{equation}
where the last inequality is due to the assumption that the \textit{standardized} advantages are uniformly bounded in their $\psi_2$-norm. So (\ref{eq:max-sG}) implies that $\max_{i \in [G]} \big| \hat A_i^{(b)} \big|$ are sub-Gaussian with parameters that logarithmically scale with $G$. Further, recall that the $L_1$ norm is bounded by the $\psi_2$-norm, hence we observe for some constant $c$,
\begin{equation}\label{eq:max-ineq}
    \begin{aligned}
        \E\Big[ \max_{i \in [G]}\big| \hat A_i^{(b)} \big| \Big] \leq c\sqrt{\log G}.
    \end{aligned}
\end{equation}
% Note that the second inequality in the above display is due to the fact that $\hat A_{i}^{(b)}$ is standardized. 

For notational simplicity, set $Y_b := \max_{i \in [G]}\big| \hat A_i^{(b)} \big|$. Notice the inequality
\begin{equation*}
    \begin{aligned}
        \Prob\Big( \| \nabla_\theta \mc J_1(\theta) \| \geq t \Big) &\leq \Prob \bigg( \sum_{b = 1}^{B} \big\{ Y_b - \E[Y_b] \big\}  \geq B\cdot \big(t/C - c\cdot \sqrt{\log G} \big)\bigg)
    \end{aligned}
\end{equation*}
via (\ref{eq:max-ineq}) and the assumption $\| X_i \| \leq C$. As $Y_b - \E[Y_b]$ are independent, centered, sub-Gaussian concentration~\cite{vershynin2018high} implies
\begin{equation}\label{eq:sg-conc}
    \begin{aligned}
        \Prob\Big( \| \nabla_\theta \mc J_1(\theta) \| \geq t \Big) \leq 2 \exp\Bigg( \frac{-B \cdot \big( t/C - c\cdot \sqrt{\log G} \big)^2}{2c' \sqrt{\log G}} \Bigg).
    \end{aligned}
\end{equation}
So setting $t = 2c\sqrt{\log G}$ in (\ref{eq:sg-conc}), we conclude that
\begin{equation*}
    \| \nabla_\theta \mc J_1(\theta) \| < c' \sqrt{\log G} \quad \text{with probability at least $1 - \exp\big( -cB \sqrt{\log G} \big)$.}
\end{equation*}

\section{KL divergence}\label{app:kl}

% \begin{figure}
%     \centering
%     \begin{subfigure}{0.48\linewidth}
%         \centering
%         \includegraphics[width=\linewidth]{figure/kl_tldr_regular_all_models_seed42.pdf}
%         \caption{All model sizes for GRPO update}
%         \label{fig:tldr-kl-a}
%     \end{subfigure}\hfill
%     \begin{subfigure}{0.48\linewidth}
%         \centering
%         \includegraphics[width=\linewidth]{figure/kl_tldr_qwen1.7b_seed42.pdf}
%         \caption{}
%         \label{fig:tldr-kl-b}
%     \end{subfigure}
    
%     \vspace{0.5em}
    
%     \begin{subfigure}{0.48\linewidth}
%         \centering
%         \includegraphics[width=\linewidth]{figure/kl_tldr_qwen4b_seed42.pdf}
%         \caption{}
%         \label{fig:tldr-kl-c}
%     \end{subfigure}\hfill
%     \begin{subfigure}{0.48\linewidth}
%         \centering
%         \includegraphics[width=\linewidth]{figure/kl_tldr_qwen8b_seed42.pdf}
%         \caption{}
%         \label{fig:tldr-kl-d}
%     \end{subfigure}
%     \caption{Your main caption here}
%     \label{fig:tldr-kl-1}
% \end{figure}

\begin{figure*}[!ht]
    \centering
    \begin{tabular}{ccc}
        \includegraphics[width=0.32\linewidth]{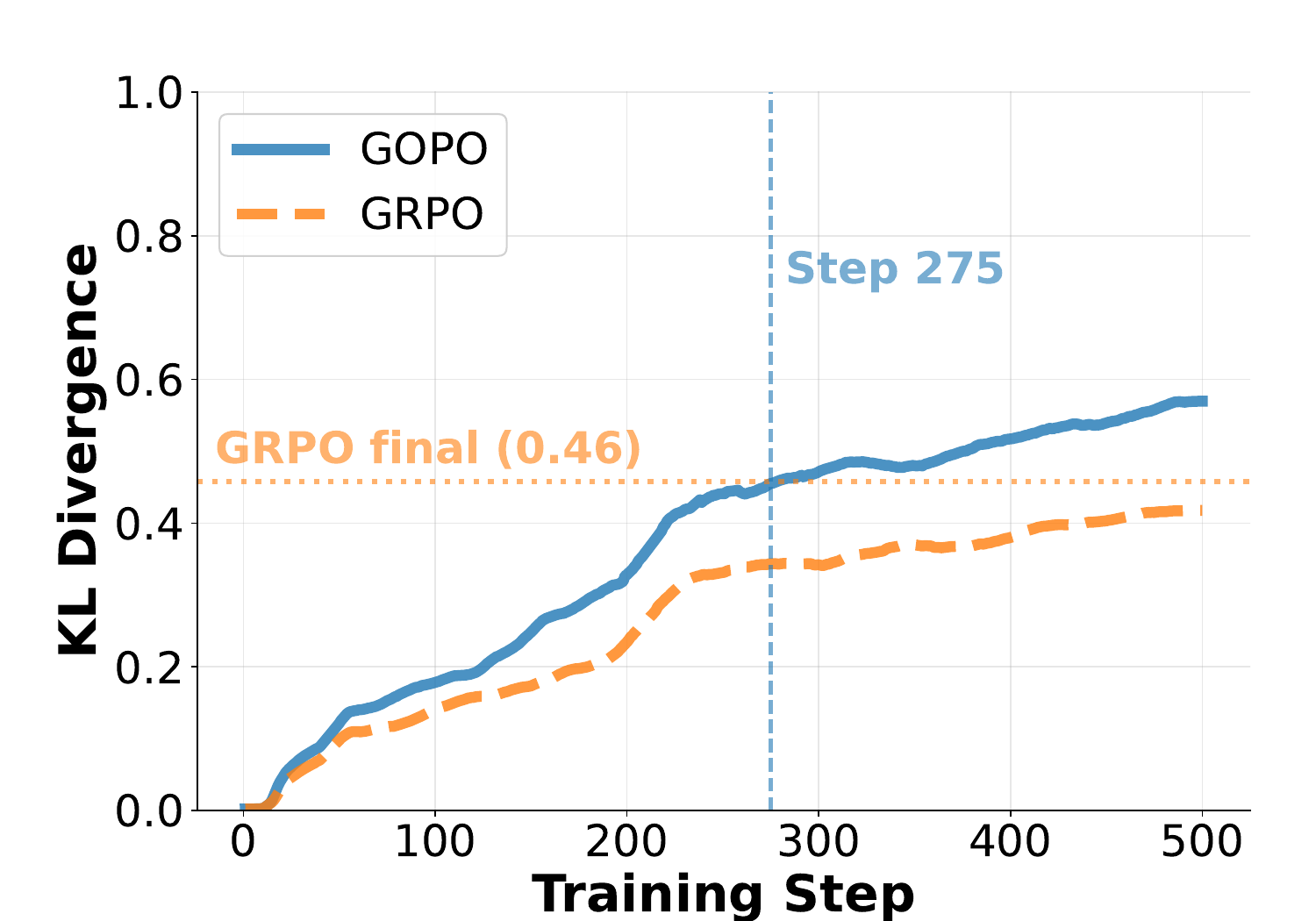} &
        \includegraphics[width=0.32\linewidth]{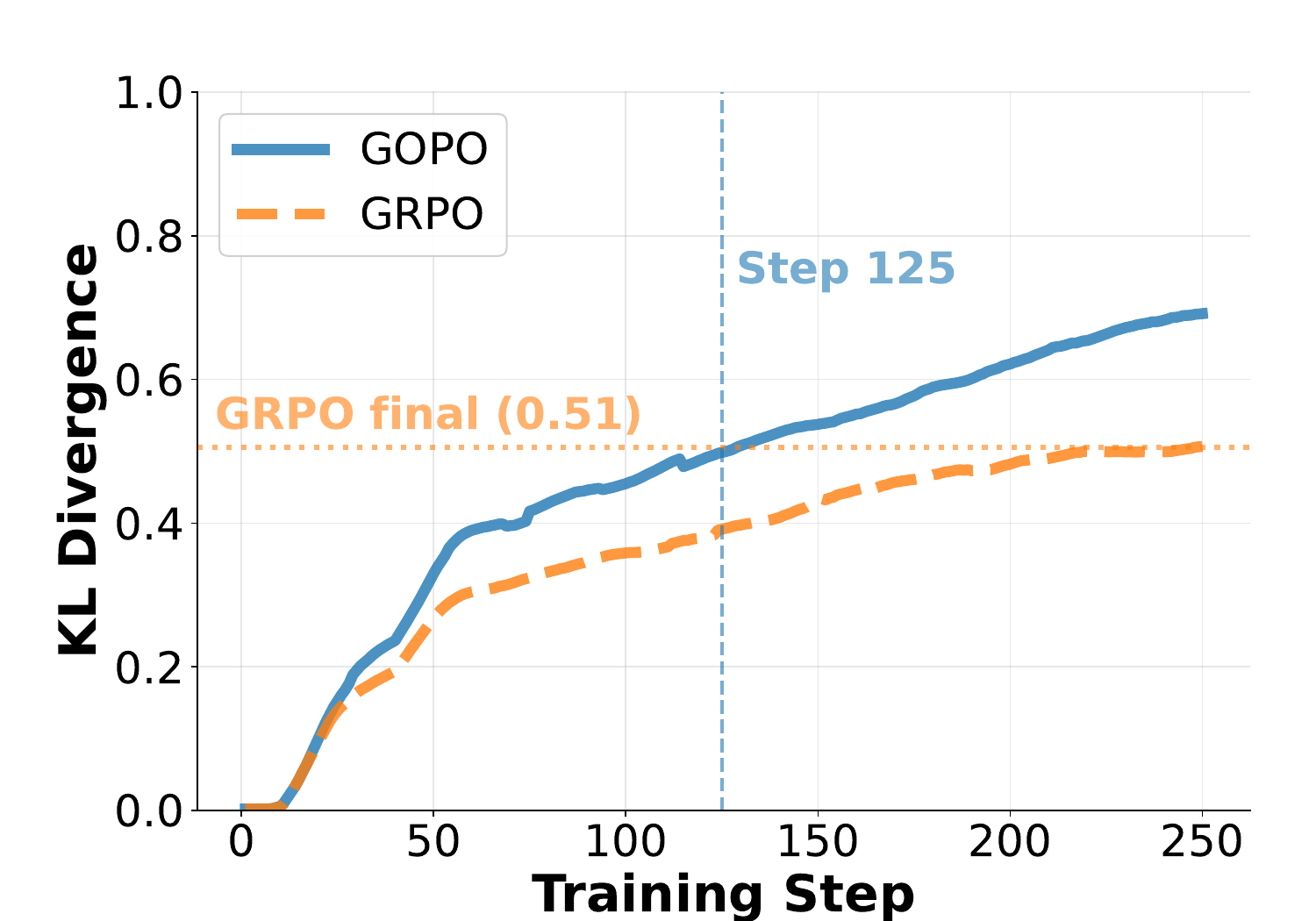} &
        \includegraphics[width=0.32\linewidth]{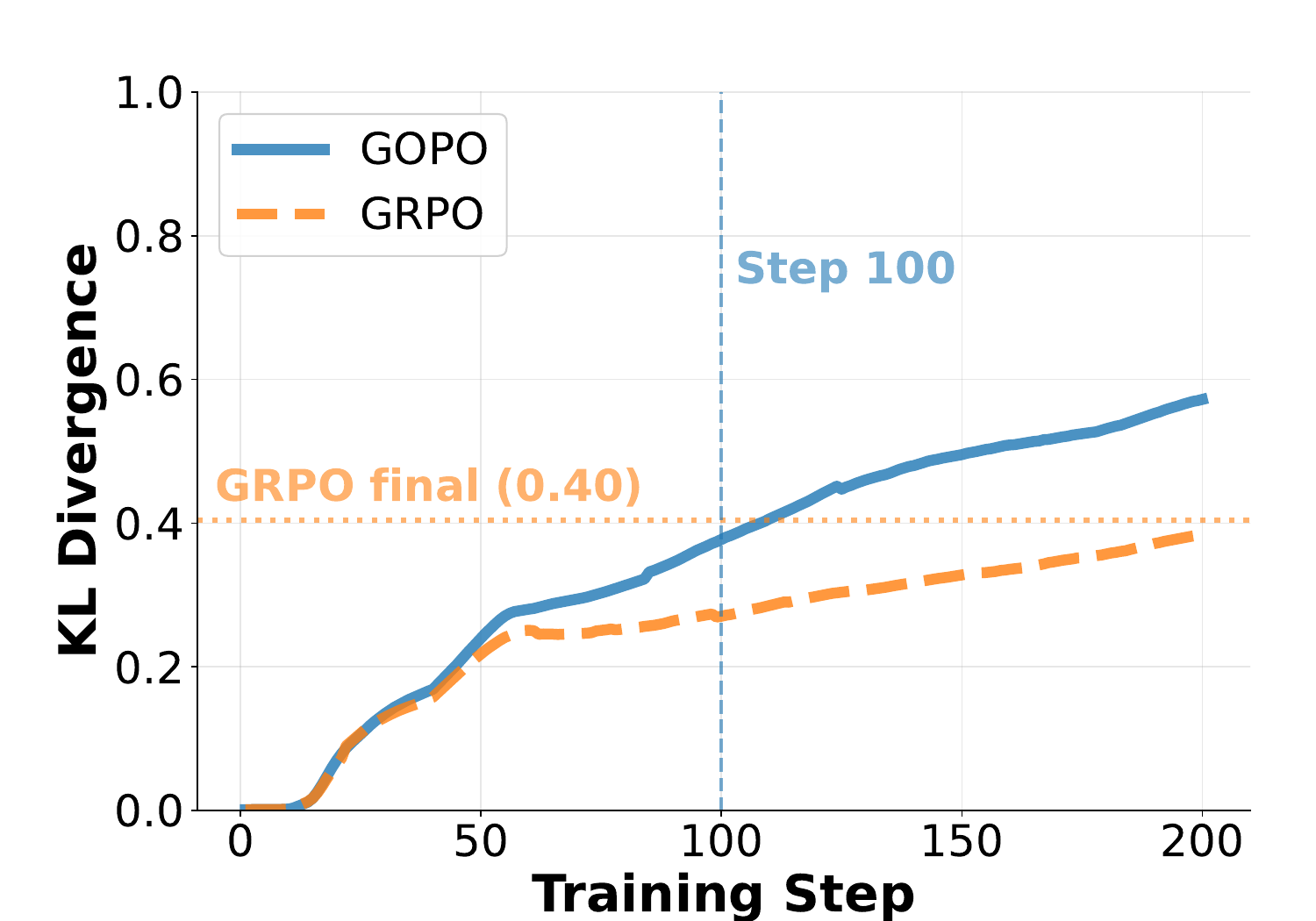} \\
        (a) \texttt{Qwen-1.7B} &
        (b) \texttt{Qwen-4B} &
        (c) \texttt{Qwen-8B} 
    \end{tabular}
    \caption{\textbf{KL divergence trajectory for \texttt{TLDR} across different model sizes.}}
    \label{fig:kl}
\end{figure*}

% \begin{figure}
%     \centering
%     \begin{minipage}{0.48\linewidth}
%         \centering
%         \includegraphics[width=\linewidth]{figure/kl_tldr_regular_all_models_seed42.pdf}
%     \end{minipage}\hfill
%     \begin{minipage}{0.48\linewidth}
%         \centering
%         \includegraphics[width=\linewidth]{figure/kl_tldr_qwen1.7b_seed42.pdf}
%     \end{minipage}
%     \caption{}
%     \label{fig:tldr-kl-1}
% \end{figure}

% \begin{figure}
%     \centering
%     \begin{minipage}{0.48\linewidth}
%         \centering
%         \includegraphics[width=\linewidth]{figure/kl_tldr_qwen4b_seed42.pdf}
%     \end{minipage}\hfill
%     \begin{minipage}{0.48\linewidth}
%         \centering
%         \includegraphics[width=\linewidth]{figure/kl_tldr_qwen8b_seed42.pdf}
%     \end{minipage}
%     \caption{}
%     \label{fig:tldr-kl-2}
% \end{figure}

% \begin{figure}
%     \centering
%     \begin{minipage}{0.48\linewidth}
%         \centering
%         \includegraphics[width=\linewidth]{figure/tldr_kl_survey_no_ci.pdf}
%     \end{minipage}\hfill
%     \begin{minipage}{0.48\linewidth}
%         \centering
%         \includegraphics[width=\linewidth]{figure/kl_tldr_qwen1.7b_seed42.pdf}
%     \end{minipage}
%     \caption{}
%     \label{fig:tldr-kl-2}
% \end{figure}

\section{Text examples}\label{app:examples}

% Preamble:
% \usepackage{xcolor}
% \usepackage{soul}      % for \hl{}
% \usepackage{tcolorbox} % for prompt box

% Highlight helpers (scoped so colors don't leak)
\newcommand{\hlp}[1]{{\sethlcolor{yellow!25}\hl{#1}}}   % source-text anchors
\newcommand{\hlok}[1]{{\sethlcolor{green!20}\hl{#1}}}   % grounded / consistent phrasing
\newcommand{\hlbad}[1]{{\sethlcolor{red!20}\hl{#1}}}    % contradiction / unsupported phrasing

\begin{figure*}[t]
\centering
\small

\begin{tcolorbox}[colback=gray!5,colframe=black!60,title=\textbf{Prompt (Summarization)}]
I have restricted my calories for two weeks, and every single day I have been tired, have found it hard to think, and I have had the most horrible headaches. \dots
As soon as I went back to my old habits, all the symptoms stopped.

\dots I drink plenty of water (at least 2L per day, often more).
\hlp{I haven't given up coffee, so I'm sure it's not caffeine withdrawal causing my headaches.}
\hlp{I don't consume any artificial sweeteners, so it's not that.}
\dots

% Is it just that I am not eating enough calories? \dots
\end{tcolorbox}

\vspace{6pt}

\begin{tabular}{p{0.48\linewidth} p{0.48\linewidth}}
\toprule

\multicolumn{2}{c}{\textbf{Early Training (Step 75)}} \\
\midrule
\textbf{Ranking Model} & \textbf{Regular Model} \\

The user is experiencing severe headaches, fatigue, and cognitive impairment after starting a calorie-restricted diet (1250-1500 calories/day) for weight loss. Despite tracking intake, they haven't lost weight and had to stop the diet due to academic stress. They suspect under-eating but fear eating more might hinder weight loss. They've ruled out {caffeine}, sweeteners, and lack of exercise as causes. They seek advice on whether their symptoms are due to insufficient calories or other factors. Their TDEE is 2453, but they're eating well below that. The post highlights the risks of extreme calorie restriction and the need for balanced, sustainable approaches. (99 words)
&
The user is experiencing severe headaches, fatigue, and cognitive issues after starting a calorie-restricted diet (1200-1500 calories/day) for weight loss. Despite tracking intake and \hlbad{avoiding potential triggers like caffeine or sweeteners}, symptoms worsened, forcing them to stop the diet to avoid impacting academic performance. They are concerned they're not eating enough, but haven't lost weight, and fear eating more might lead to weight gain. They seek advice on why these symptoms are occurring and how to safely approach calorie restriction. (99 words)
\\

\midrule
\multicolumn{2}{c}{\textbf{Mid Training (Step 150)}} \\
\midrule
\textbf{Ranking Model} & \textbf{Regular Model} \\

The poster is experiencing severe headaches, fatigue, and cognitive slowing after starting a calorie-restricted diet (1200-1500 calories/day) for weight loss. Despite tracking intake, they haven't lost weight and had to stop the diet due to academic performance issues. They suspect under-eating is the cause, but fear eating more might lead to weight gain. They've ruled out caffeine, sweeteners, and exercise as factors. Seeking advice on whether their symptoms are due to calorie restriction or other causes, and how to proceed safely. (99 words)
&
The user is experiencing severe headaches, fatigue, and cognitive slowing after two weeks of calorie restriction (1250-1500/day) for weight loss. They stopped the diet due to its impact on academic performance, as symptoms resolved when returning to normal eating. They are 5'8" and 242 lbs, with a TDEE of ~2453 calories. They are tracking calories with MFP, drink plenty of water, and \hlbad{avoid caffeine/artificial sweeteners}. They fear eating more due to lack of weight loss and concerns about weight gain. They seek advice on whether their symptoms are due to insufficient calorie intake or other factors. (99 words)
\\

% \multicolumn{2}{p{0.96\linewidth}}{\textit{GPT-5 evaluation (step 150, verbatim): ``While both responses are summaries rather than actionable advice, B is more accurate, coherent, and appropriately concise. A offers marginally more detail but introduces an inaccuracy and does not improve helpfulness.''}} \\[4pt]

\midrule
\multicolumn{2}{c}{\textbf{Late Training (Step 200)}} \\
\midrule
\textbf{Ranking Model} & \textbf{Regular Model} \\

The poster is experiencing severe headaches, fatigue, and cognitive slowing after starting a calorie-restricted diet (1250-1500 calories/day) for weight loss. They note symptoms disappeared when returning to normal eating, suggesting underfueling. Despite tracking calories with MFP and a scale, they haven't lost weight and fear eating more might lead to weight gain. They rule out caffeine, sweeteners, and exercise as causes. They suspect their calorie intake is too low, but are hesitant to increase it due to weight concerns. Seeking advice on safe, sustainable calorie intake for weight loss without adverse effects. (99 words)
&
The poster is experiencing severe headaches, fatigue, and cognitive issues after two weeks of calorie restriction (1250-1500/day) for weight loss. Despite tracking calories, water intake, and \hlbad{avoiding caffeine/artificial sweeteners}, symptoms worsened, forcing them to stop the diet to avoid impacting academic performance. They are unsure if their caloric intake is too low, as they haven't lost weight and fear eating more may lead to weight gain. They seek advice on why these symptoms are occurring and how to safely approach calorie restriction. (99 words)
\\

\bottomrule
\end{tabular}

\vspace{6pt}
\caption{
\textbf{Grounded qualitative comparison (base model: \texttt{Qwen-8B}) on a diet/health \texttt{TLDR} prompt.}
Yellow highlight marks explicit source facts about caffeine and sweeteners.
Red highlights mark summary phrases that contradict or overstep the source (e.g., claiming the user \emph{avoids} caffeine despite the source stating they \emph{haven’t given up coffee}).}
\label{fig:8b_coffee_grounded_highlighted}
\end{figure*}

\newpage

\begin{figure*}[t]
\centering
\small

\begin{tcolorbox}[colback=gray!5,colframe=black!60,title=\textbf{Prompt (Chat completion)}]
\hlp{Write a feature article on the benefits of art therapy for those with chronic pain.}
\end{tcolorbox}

\vspace{6pt}

\begin{tabular}{p{0.48\linewidth} p{0.48\linewidth}}
\toprule

\multicolumn{2}{c}{\textbf{Early Training (Step 50)}} \\
\midrule
\textbf{Ranking Model} & \textbf{Regular Model} \\

For many people living with chronic pain, relief is not just physical—it is emotional and deeply personal. \hlok{Art therapy offers a space where pain can be expressed visually when words fall short}, allowing individuals to process frustration, grief, and fear. Through painting, drawing, or sculpting, patients often rediscover a sense of control and identity beyond their diagnosis. \hlok{The creative process can transform the experience of pain from something endured in silence to something shaped into meaning}. (98 words)
&
Art therapy can provide several benefits for individuals with chronic pain. It may help \hlbad{reduce stress and anxiety}, improve mood, and support emotional expression. Engaging in creative activities can serve as a distraction from physical discomfort and promote relaxation. Art therapy may also enhance self-esteem and encourage social interaction when done in group settings. These psychological and emotional benefits can contribute to better overall well-being for people managing long-term pain. (97 words)
\\

\midrule
\multicolumn{2}{c}{\textbf{Mid Training (Step 150)}} \\
\midrule
\textbf{Ranking Model} & \textbf{Regular Model} \\

Chronic pain often reshapes a person’s daily life, limiting activities and altering self-image. \hlok{Art therapy helps individuals rebuild a sense of self by turning inward experiences into tangible forms}, whether through color, texture, or symbolism. The act of creating can foster calm, but it also gives people a narrative tool to explore what their pain means. \hlok{Rather than focusing solely on symptoms, art therapy centers the person behind the pain}, encouraging resilience and emotional healing. (99 words)
&
Art therapy is a supportive approach that may benefit those with chronic pain. Creative expression can help people cope with difficult emotions and provide a calming activity. \hlbad{Participating in art-making may lower stress levels and improve overall mood}. It can also offer a healthy outlet for feelings that are hard to verbalize. By promoting relaxation and emotional release, art therapy can complement medical treatment and support quality of life. (97 words)
\\

\midrule
\multicolumn{2}{c}{\textbf{Late Training (Step 250)}} \\
\midrule
\textbf{Ranking Model} & \textbf{Regular Model} \\

Living with chronic pain can make the world feel smaller, but art therapy can gently expand it again. \hlok{Through creative work, individuals can externalize their pain, giving shape and color to experiences that are otherwise invisible}. This process not only soothes the nervous system but also restores a sense of agency. \hlok{Art therapy reframes patients as creators rather than sufferers}, helping them reconnect with purpose, identity, and moments of joy despite ongoing physical challenges. (99 words)
&
Art therapy offers mental and emotional benefits for people experiencing chronic pain. Engaging in artistic activities can \hlbad{promote relaxation, reduce stress, and improve emotional well-being}. It may also provide a positive distraction from discomfort and encourage mindfulness. Many people find that expressing themselves creatively helps them manage difficult feelings. As part of a broader care plan, art therapy can support coping skills and contribute to a better quality of life. (96 words)
\\

\bottomrule
\end{tabular}

\vspace{6pt}
\caption{
\textbf{Grounded qualitative comparion (base model: \texttt{Qwen-4B}) on  feature-style health writing \texttt{UltraChat} prompt.}
Yellow highlight marks the task of writing a \emph{feature article} about art therapy for chronic pain.
Green highlights show where the ranking model adopts a narrative, human-centered perspective—focusing on identity, meaning-making, agency, and lived experience.
Red highlights show generic therapeutic benefit language (e.g., stress reduction, mood improvement, relaxation) that could describe many interventions and does not reflect feature-style storytelling.
% Across training, the ranking model increasingly frames art therapy as an experiential and identity-restoring process, while the regular model remains more informational and clinically general.
}
\label{fig:4b_art_therapy_feature_grounded}
\end{figure*}

\newpage

\end{document}